%% file: main.tex
\documentclass[nohyperref]{article}

\usepackage{microtype}
\usepackage{graphicx}
\usepackage{subfigure}
\usepackage{booktabs} 

\usepackage[accepted]{sty/icml2022}
\usepackage{natbib}
\usepackage{xspace}

\usepackage{sty/notations}

\usepackage{amsmath}
\usepackage{amssymb}
\usepackage{mathtools}
\usepackage{amsthm}
\usepackage{multirow}
\usepackage{framed}
\usepackage{booktabs}
\usepackage{colortbl}
\usepackage{pifont}
\newcommand{\cmark}{\ding{51}}%
\newcommand{\xmark}{\ding{55}}%
\usepackage{enumitem}

\newcommand{\todoPi}[1]{\todo[color=yellow!40, inline]{\small Pierre: #1}}
\newcommand{\todoCo}[1]{\todo[color=green!40, inline]{\small C\^ome: #1}}
\newcommand{\todoTa}[1]{\todo[color=red!40, inline]{\small Tadashi: #1}}
\newcommand{\todoVP}[1]{\todo[color=blue!40, inline]{\small VP: #1}}
\newcommand{\todoMR}[1]{\todo[color=magenta!40, inline]{\small MR: #1}}

\theoremstyle{plain}
\newtheorem{theorem}{Theorem}[section]
\newtheorem{proposition}[theorem]{Proposition}
\newtheorem{lemma}[theorem]{Lemma}
\newtheorem{corollary}[theorem]{Corollary}
\theoremstyle{definition}

\theoremstyle{remark}
\newtheorem{remark}[theorem]{Remark}

\usepackage[textsize=tiny]{todonotes}

\usepackage{thmtools, thm-restate} 

\usepackage{minitoc}

\icmltitlerunning{Adapting to game trees in zero-sum imperfect information games}

\begin{document}

\twocolumn[
\icmltitle{Adapting to game trees in zero-sum imperfect information games}



\icmlsetsymbol{equal}{*}

\begin{icmlauthorlist}
\icmlauthor{C\^ome Fiegel}{crest}
\icmlauthor{Pierre M\'enard}{ensl}
\icmlauthor{Tadashi Kozuno}{omron}
\icmlauthor{R\'emi Munos}{deepmind}
\icmlauthor{Vianney Perchet}{crest,criteo}
\icmlauthor{Michal Valko}{deepmind}
\end{icmlauthorlist}

\icmlaffiliation{criteo}{CRITEO AI Lab, Paris, France}
\icmlaffiliation{crest}{CREST, ENSAE, IP Paris, Palaiseau, France}
\icmlaffiliation{deepmind}{Deepmind, Paris, France}
\icmlaffiliation{ensl}{ENS Lyon,  Lyon, France}
\icmlaffiliation{omron}{Omron Sinic X, Tokyo, Japan}

\icmlcorrespondingauthor{C\^ome Fiegel}{come.fiegel@normalesup.org}

\icmlkeywords{Machine Learning, ICML}

\vskip 0.3in
]



\printAffiliationsAndNotice{}  %
\doparttoc 
\faketableofcontents 

\begin{abstract}
Imperfect information games (IIG) are games in which each player only partially observes the current game state. We study how to learn $\epsilon$-optimal strategies in a zero-sum IIG through self-play with \textit{trajectory feedback}. We give a problem-independent lower bound $\widetilde{\mathcal{O}}(H(A_{\mathcal{X}}+B_{\mathcal{Y}})/\epsilon^2)$ on the required number of realizations to learn these strategies with high probability, where $H$ is the length of the game, $A_{\mathcal{X}}$ and $B_{\mathcal{Y}}$ are the total number of actions for the two players. We also propose two Follow the Regularized leader (FTRL) algorithms for this setting: \texttt{Balanced~FTRL} which matches this lower bound, but requires the knowledge of the information set structure beforehand to define the regularization; and \texttt{Adaptive~FTRL} which needs $\widetilde{\mathcal{O}}(H^2(A_{\mathcal{X}}+B_{\mathcal{Y}})/\epsilon^2)$ realizations without this requirement by progressively adapting the regularization to the observations.
\end{abstract}

\setlength{\footskip}{3.30003pt}

\input{main/introduction}

\input{main/setting}
\input{main/lower_bound}
\input{main/optimal_rate}
\input{main/practical_algorithm}
\input{main/experiments}
\input{main/conclusion}

\bibliography{refs,come_refs}
\bibliographystyle{icml2022.bst}

\onecolumn
\part{Appendix}
\parttoc
\newpage
\appendix
\input{appendix/table_notations.tex}
\input{appendix/related_work}
\input{appendix/lower_bound_tadashi}
\input{appendix/appendix_copy}


\end{document}

%% file: main/introduction.tex
\section{Introduction}
\label{sec:introduction}


In imperfect information games (IIG), players, upon taking an action, may only have access to \emph{partial information about the true current game state}. This type of games allows for the modelling of complex strategic behavior such as bluffing \citep{koller1995generating}. In this work, in particular, we study extensive-form two-players zero-sum IIGs. 

The \emph{extensive-form} of a game describes, which player is playing and which actions are available, sequentially and depending on the previous moves. It is typically represented by a tree of depth $H$, where nodes are \emph{states} that are controlled by one of the players that decides, based on each player action, the next game states among its children.

Players can be uncertain about the true game state upon playing. We therefore assume that the set of states controlled by the player is partitioned into \emph{information sets}, where an information set is a set of states that are indistinguishable to this player. As commonly assumed in game theory, we suppose \emph{perfect recall} \citep{kuhn1950extensive}: players remember their previous moves, which implies that the space of information sets has a \emph{tree structure}. 

We focus more specifically on \emph{zero-sum games} (what the max-player gains is the opposite of the other, min-player) and aim at devising algorithms that learn $\epsilon$-optimal strategies \citep{neumann1928zur}. For this purpose, we actually develop a unilateral algorithm that dictates to a single player how to play over multiple realizations of the game called \emph{episodes}, each of the same length $H$, in order to minimize the \emph{regret}---the difference between the cumulative gain and the gain of the best fixed strategy---against any sequence of moves of the adversary. If both players follows such an algorithm, we can then prove that their average strategy becomes more precise over time. We especially want to construct algorithms with a regret that scales as slowly as possible with respect to the different game parameters.

For any information set $x\in \cX$, we define by $\Ax$ the number of available actions at $x$, as two indistinguishable states must have the same number of actions. The total number of actions of the max-player is then defined by $\AX=\sum_{x\in\cX} \Ax$. Similarly, we let $\cY$ be the collection of information sets of the min-player, $\By$ the number of actions available at $y\in\cY$ and $\BY=\sum_{y\in \cY}\By$ the total number of actions. Meanwhile $X$ and~$Y$ denote the number of information sets for the max and min-players.

\paragraph{Related work} When the structure of the game, transition probabilities and reward function are known beforehand, several methods exist to approximate the optimal strategy. A common approach is to reformulate the problem as a linear program that can be solved efficiently \citep{Rom62,VONSTENGEL1996220,koller1996efficient}. Some methods see it as a saddle-point problem and rely on first-order optimization methods \citep{hoda2010smoothing, kroer2015faster,kroer2018solving,kroer2020faster,munos2020fast,lee2021last}. Other common approaches try to locally minimize the regret at each information set, e.g., \textit{counterfactual regret minimization} \citep{zinkevich2007regret,tammelin2014solving,burch2019revisiting}.

The main practical weakness of the above approaches is their prohibitive cost as the size of the game grows. Indeed, the computations are mostly done on the whole information set trees, which implies a time complexity of at least $\cO(X+Y)$ at each step at best \citep{lanctot2009monte}. All the recent successes in empirically solving large IIGs actually avoid the direct use of these full information algorithms \citep{moravcik2017deepstack,brown2018superhuman,schmid2021player,bakhtin2022human,perolat2022mastering}. In addition, depending on the practical case, it may not be possible to know the state transitions, or even to start the game at an arbitrary state.

We therefore treat the settings when the algorithms can only access a \emph{trajectory feedback} of the game: they only have access to the observations of the player. For this case, the outcome sampling \MCCFR \citep{lanctot2009monte,shcmid2018variance, farina2020stochastic} feeds estimates of the counterfactual regret to the \CFR algorithm, which is by design a full-information feedback algorithm. Precisely, outcome-sampling \MCCFR uses a uniform sampling of the actions in each information set, which is actually sub-optimal when the information set trees of the players are unbalanced, as information sets in smaller sub-trees are then observed way more than information sets of bigger sub-trees. For this reason, \citet{farina2020stochastic} instead propose to build an estimate of the counterfactual regret according to a \emph{balanced policy} that, roughly speaking, samples actions proportionally to the size of the associated sub-trees.\footnote{The exact definition of the balanced policy varies among different works, but the main idea remains the same.} While this balanced policy allows for a better sample complexity, it also requires to know the game structure beforehand. 

Closer to the adversarial bandits, \citet{farina2021bandit} propose to use the \emph{online mirror descent} (\OMD) algorithm along with a dilatation of the Shannon entropy on all information sets acting as a regularizer \citep{hoda2010smoothing}. However, their approach uses importance sampling \citep{auer2003nonstochastic} to estimate the losses, whose variance is too large to allow the computation of $\epsilon$-optimal strategies with high probability. \citet{kozuno2021learning} solved this issue by using a biased estimate called \emph{Implicit exploration} (IX, \citealp{kocak2014efficient,neu2015explorea,lattimore2020bandit_book}), and proved that their \IXOMD algorithm achieves a sample complexity of order $\tcO(H^2(X\AX+Y\BY  )/\epsilon^2)$ with high probability\footnote{For algorithms with a probability at least $1-\delta$ of a correct output, the symbol $\tcO$ hides dependencies logarithmic in $H,\AX,\BY,\delta$ and $\epsilon$.}. Unfortunately, it does not match the $\tcO(H(\AX+\BY)/\epsilon^2)$ lower bound, see Table~\ref{tab:sample_complexity}. 

More recently, \citet{bai2022nearoptimal} propose to weight differently each of the information set in the regularizer using the aforementioned balanced policy. They obtain a sample complexity of $\tcO(H^3(\AX+\BY)/\epsilon^2)$ that scales linearly with the total number of actions, but again, at the extra price of requiring the knowledge of the game structure beforehand.

Hence, we raise the following questions: 
\emph{What is the optimal rate for learning $\epsilon$-optimal strategies through self-play?}
 \emph{Is it possible to learn these $\epsilon$-optimal strategies with a complexity scaling linearly with the total number of actions without the prior knowledge of the information set tree structure?} 
In this work we answer \emph{both questions}.

\paragraph{Contributions} We make the following contributions:

\begin{itemize}[ itemsep=-2pt,leftmargin=6pt]
\item We prove a lower bound of $\cO(H(\AX+\BY)/\epsilon^2)$ on the number of necessary plays required to compute an $\epsilon$-optimal profile with only a trajectory feedback. Compared to the one stated by \citet{bai2022nearoptimal}, it applies to algorithms that are only correct with a high probability, and enjoys an extra $H$ factor if the size of the action set is allowed to vary with the information set.

\item We propose the \BalancedFTRL algorithm that matches this lower bound, up to logarithm factors, using the knowledge of the structure of the information sets. It uses the Follow The Regularized Leader (\FTRL) algorithm instead of \OMD and follows the idea of balancing, but using instead a concept of \emph{balanced transition}. In addition to the use of Shannon entropy, it also allows the use of the Tsallis entropy, in order to get a better sample complexity at the price of increased computation time.

\item We then propose the \AdaptiveFTRL algorithm which matches this lower bound up to an extra $H$ and logarithmic factors, \emph{without} using the knowledge of the structure. Instead of using the balancing transition, it estimates the \emph{actual} transitions of each episode.

\item We provide a practical implementation of  \BalancedFTRL and \AdaptiveFTRL through Algorithm~\ref{al:adapt_update} of Appendix~\ref{app:algorithmic_update}. This implementation is computationally efficient, as the policy is only updated on the information sets visited in the previous episode.

\item Finally, we provide experiments that compare the performances of these two algorithms with \IXOMD and \BalancedOMD. We observe that, despite having different theoretical guarantees, the algorithms all seem to have comparable performances in practice.

\end{itemize}

All rates are summarized in Table~\ref{tab:sample_complexity}.

\begin{table*}[t]
\centering
\label{tab:sample_complexity}
\begin{tabular}{@{}lcc}\toprule
\textbf{Algorithm}  & \textbf{Sample complexity} & \textbf{Structure-free}\\
\midrule
\MCCFR~{\scriptsize \citep{farina2020stochastic,bai2022nearoptimal}}& $\tcO(H^4(\AX+\BY )/\epsilon^2)$ &  \xmark \\
\IXOMD~{\scriptsize \citep{kozuno2021learning}} &  $\tcO(H^2(X\AX+Y\BY)/\epsilon^2)$ & \cmark\\
\BalancedOMD~{\scriptsize \citep{bai2022nearoptimal}} & $\tcO(H^3(\AX+\BY)/\epsilon^2)$ & \xmark \\
\midrule 
 \rowcolor[gray]{.90} \BalancedFTRL~{\scriptsize(this paper)} & $\tcO(H(\AX+\BY)/\epsilon^2)$ & \xmark \\
 \rowcolor[gray]{.90} \AdaptiveFTRL~{\scriptsize(this paper)} &  $\tcO(H^2(\AX+\BY)/\epsilon^2)$ & \cmark\\
\midrule 
Lower bound~{\scriptsize(this paper)}& $\tcO(H(\AX+\BY)/\epsilon^2)$& \\
\bottomrule
\end{tabular}
\caption{Sample complexity for episodic, finite, two-player, zero-sum IIGs. \textit{Structure-free} algorithm designates an algorithm that does not need to know  the structure of the information set spaces in advance. The symbol $\tcO$ hides  dependencies logarithmic in  $\AX,\BY,\epsilon$ and $\delta$. Note that for algorithms that only work with fixed action sets of size $A$ and $B$ we have $\AX=AX$ and $\BY=BY$.}
\end{table*}

%% file: main/setting.tex
\section{Setting}
\label{sec:setting}
 
Consider an episodic, finite, two-player, zero-sum IIG {\small$(\cS,\cX,\cY,\{\cA(x)\}_{x\in \cX},\{\cB(y)\}_{y\in \cY},H,p,r)$}, which consists of the following components \citep{kuhn1950extensive, littman1994markov}:
\begin{itemize}[ itemsep=-2pt,leftmargin=6pt]
\item A finite state space~$\cS$ and two information set spaces (partitions of $\cS$) $\cX$ of size $X$ and $\cY$ of size $Y$ for the max- and min-player respectively. 
\item For each $x\in\cX$ and $y\in\cY$, finite action spaces $\cAx$ of size $\Ax$ and $\cBy$ of size $\By$. 
\item The length $H \in \N$ of the game. 
\item Initial state distribution $p_0 \in \Delta (\cS)$ and a state-transition probability kernel $(p_h)_{h\in[H-1]}$ with\footnote{In many games with tree structure, $p_h$ does not depend on the step~$h,$ but we keep the dependence on $h$ as it makes the results more general at no extra cost in the analysis.}
$p_h: \cS \times \cA(\cX) \times \cB(\cY) \rightarrow \Delta(\cS)$ for each $h \in [H-1]$. 
\item A reward function $(r_h)_{h\in[H]}$ with $r_h: \cS \times \cA(\cX) \times \cB(\cY) \rightarrow [0,1]$. 
\end{itemize}



\paragraph{Perfect-recall} 

As explained in the introduction, we assume perfect-recall. Formally, this means that for the max-player and for each information set $x\in\cX$ there exists a unique $h\in[H]$ and history $(x_1,a_1, ... ,x_h)$ such that $x_h=x$. Specifically, we write $x\geq x'$ if $x'$ is part of the history that leads to $x$. With this assumption, both $\cX$ and $\cY$ can be partitioned into $H$ different subsets $(\cX_h)_{h\in[H]}$ and $(\cY_h)_{h\in[H]}$, $\cX_h$ and $\cY_h$ being the sets of possible information sets at time step $h$ for respectively the max and min-player.

For convenience we let $\cAXh:= \{(x_h, a_h)\!:\ x_h\in\cX_h\text{ and }a_h\in\cA(x_h)\}$ be the total action set for the max-player at depth $h$, and use an analogous notation $\cBYh$ for the min-player. The unions of these sets are denoted by $\cAX:= \cup_{h\in[H]} \cAXh,\, \cBY:= \cup_{h\in[H]} \cBYh$ with their respective sizes $\AX$ and $\BY$, as defined in the introduction.




\paragraph{Policies} The perfect-recall assumption allows us to represent a policy of the max-player as a sequence $\mu=(\mu_h)_{h \in [H]}$ such that for all $x_h\in \cX_h$, $\mu_h(.|x_h)$ is an element of $\Delta\pa{\cAs{x_h}}$. A policy $\nu$ of the min-player can be defined similarly. We let $\maxpi$ and $\minpi$ be the sets of the max- and min-player's policies.

\paragraph{Episode unfolding} Given the two players' policies $\mu$ and~$\nu$, an episode of the game proceeds as follows: an initial state $s_1\sim p_0$ is sampled. At step $h$, the max- and min-player observe their information sets $x_h$ and $y_h$. Given the information, the max- and min-player choose and execute actions $a_h \sim \mu_h (\cdot | x_h)$ and $b_h \sim \nu_h (\cdot | y_h)$. As a result, the current state transitions to a next state $s_{h+1} \sim p_h (\cdot | s_h, a_h, b_h)$, and the max- and min-player receive rewards $r_h (s_h, a_h, b_h)$ and $-r_h(s_h,a_h,b_h)$ respectively. This is repeated until time step $H$, after which the episode finishes.

\paragraph{Learning procedure} We assume there are $T$ episodes of the same game, and both players are able to progressively adapt their respective policies $\mu^t$ and $\nu^t$ after each episode based on their previous observations.\footnote{As in \citet{kozuno2021learning,bai2022nearoptimal}, we assume that the players can not change their policies within an episode, as the tree-like structure limits the interest of such adaptation.}

\paragraph{Realization plan and loss} Given the perfect recall assumption, we recursively define  \citep{VONSTENGEL1996220} the realization plan $\mu_{1:}:=(\mu_{1:h})_{h\in [H]}$ by, given $(x_1,a_1, ... , x_h)$ the unique history up to $x_h$, 
\[\mu^{}_{1:h}(x_h,a_h):=\prod_{h'=1}^h \mu_{h'}(a_{h'}|x_{h'})\,.\]
We also define the adversarial transitions $p^\nu_{1:}:=(p^\nu_{1:h})_{h\in[H]}$ and losses $\ell^\nu:=(\ell^\nu_h)_{h\in[H]}$ by:
\begin{align*}
    p_{1:h}^\nu(x_h) &:=p^{}_0(x_1) \prod_{h'=1}^{h-1} p_{h'}^{\nu}(x_{h'+1}|x_{h'},a_{h'})\,,\\
    \ell^{\nu}_h(x_h,a_h)&:= p^\nu_{1:h}(x_h)\pa{1-r_h^\nu(x_h,a_h)},
\end{align*}
where $p_{h}^{\nu}(x_{h+1}|x_{h},a_{h})$ is the probability to transition to the information state $x_{h+1}$ from the pair information set-action $(x_{h},a_{h})$ when the opponent policy is fixed to $\nu$, and $r_h^\nu(x_h,a_h)$ is the average reward obtained in this case. Note that these two quantities require the perfect recall assumption in order to be well defined, and combine the randomness of both state transitions and min-player's policy $\nu$.

The analog quantities $\nu_{1:}$, $p^\mu_{1:}$, $r^\mu_h$ and $\ell^\mu$ may also be defined for the min-player, replacing $1-r^\nu_h$ by $r^\mu_h$ in the definition of the loss.

\paragraph{Regret and $\varepsilon$-Nash-equilibrium} 

We define the expected rewards (of the max-player) $V^{\mu, \nu} :=\E^{\mu,\nu}\!\left[\sum_{h=1}^H r_h\right]$ for a pair $(\mu,\nu)\in \maxpi\times\minpi$ of policies and, after $T$ episodes, the regrets of the max and min players,
\begin{align*}
    \regret^T_{\mathrm{max}}:&= \max_{\mu^\dagger \in \maxpi} \sum_{t=1}^T \pp{V^{\mu^\dagger, \nu^t} - V^{\mu^t, \nu^t}}\,,\\
    \regret^T_{\mathrm{min}}:&= \max_{\nu^\dagger \in \minpi} \sum_{t=1}^T \pp{V^{\mu^t, \nu^t} - V^{\mu^t, \nu^\dagger}}\,.
\end{align*}

There is a very clear link between minimizing the regret of both players and computing $\epsilon$-optimal profiles, or $\varepsilon$-Nash-equilibrium (NE). In the particular case of perfect-recall, Theorem~1 by \citet{kozuno2021learning} states the following (see also \citealt{cesa-bianchi2006prediction,zinkevich2007regret}).

\begin{theorem} \label{thm:folklore}
    From $(\mu^t,\nu^t)_{t\in[T]}$ define the time-averaged profile $(\overline{\mu},\overline{\nu})$ (of the realization plan, see Appendix~\ref{appendix:bal}), then $(\overline{\mu},\overline{\nu})$ is $\epsilon$-optimal with
    \[\varepsilon=\pa{\regret^T_{\mathrm{max}}+\regret^T_{\mathrm{min}}}/T\,.\]
\end{theorem}

\paragraph{Conversion to online linear regret minimization}
Defining a scalar product between the realization plans and the losses,
\[\scal{\mu_{1:}}{\ell^\nu}:=\sum_{h=1}^H\sum_{(x_h,a_h)\in\cAXh}\mu_{1:h}^{}(x_h,a_h)\ell^\nu_h(x_h,a_h),\]
the chain rule on conditional probabilities implies that the max-player regret can then be rewritten, with $\ell^t:=\ell^{\nu^t}$, as
\[\regret^T_{\mathrm{max}}= \max_{\mu^\dagger \in \maxpi} \sum_{t=1}^T \scal{\mu_{1:}^t-\mu_{1:}^\dagger}{\ell^t}.\]
This converts the problem to a linear regret problem, where the constraints set (the set of realization plans) is a polytope. 

For simplicity, in the following sections we focus on the max-player  although the exact same ideas apply to the min-player, who aims at minimizing $\regret^T_{\min}$.

%% file: main/lower_bound.tex
\section{Lower bounds}
\label{sec:lower_bound}

In this section we state lower bounds for learning an approximate NE. Compared to \citet{bai2022nearoptimal}, our lower bounds cover high-probability guarantees instead of in expectation, and are stated for two settings, whether or not the size of the action set is fixed. As shown in Section~\ref{sec:tree_structure_known}, they are tight up to logarithmic factors, which justifies the existence of two different rates for these settings.

\begin{restatable}{theorem}{lbtree}\label{thm:lb}(informal, exact statement in Appendix~\ref{app:lower_bound})

    \underline{Regret:} For any algorithm that controls the max-player, there exists a game such that, if the number of action per information set is allowed to vary, with probability at least~$\delta$, the algorithm suffers a regret of
    \[\Omega\pa{\sqrt{H\AX T\log\!\left(1/(4\delta)\right)}}.\]
    Otherwise if the number of the actions is fixed, then the regret is
    \[\Omega\pa{\sqrt{\AX T\log\!\left(1/(4\delta)\right)}}.\]

    \underline{Sample complexity:} For any algorithm that controls both players, there exists a game such that, if the number of action is allowed to vary, it needs at least
    \[\Omega\pa{H(\AX+\BY)\log\!\left(1/(4\delta)\right)/\epsilon^2}\]
    episodes to output an $\epsilon$-optimal strategy with a probability at least $1-\delta$.
    
    If the number of actions is fixed, then it instead needs
    \[\Omega\pa{(\AX+\BY)\log\!\left(1/(4\delta)\right)/\epsilon^2}\ \text{episodes.}\]

\end{restatable}

These lower bounds are obtained by a reduction to the $K$-armed bandits. The structure depends on the setting: when the number of actions is allowed to vary, the learner must only choose between $K$ actions on the first information set, and then gets $H$ times the same reward depending on this first choice. When the number of actions is fixed, the learner must take $H$ actions successively (among the $A$ proposed at each information set) and only receive a reward after the last one depending on all of its choices; hence with $K=A^H$ possibilities.

%% file: main/optimal_rate.tex
\section{Optimal rate with the knowledge of the tree structure}
\label{sec:tree_structure_known}

We first study the case where the information set structure is known by the agent beforehand.

\paragraph{Loss estimation} The agent only observes the realizations of the games. This suggests that the loss vector $\ell^t$ should be estimated using only the information on the visited states and the associated rewards. Therefore, we define, for all possible actions, the unbiased loss estimate
\[\widehat{\ell}_h^{\,t}(x_h,a_h):=\frac{\indic{x_h^t=x_h,a_h^t=a_h}}{\mu_{1:h}^t(x_h,a_h)}(1-r_h^t),\]
where $x_h^t$ $a_h^t$ and $r_h^t$ are respectively the information set, action and reward of the max-player at time step $h$ of episode $t$. 

While this estimator works well in expectation, its variance is too large to get acceptable bounds with high probability. We replace it by the following (biased) estimator $\tell^t$, with an implicit exploration (IX) parameter $\gamma_h$ that depends on the chosen state \citep{kocak2014efficient}:
\[\tell_h^t(x_h,a_h):=\frac{\indic{x_h^t=x_h,a_h^t=a_h}}{\mu_{1:h}^t(x_h,a_h)+\gamma_h(x_h)}(1-r_h^t).\]
The associated cumulative loss is given as $\tL^t:=\sum_{k=1}^t\tell^{k}$.

\paragraph{Balanced transitions} 

We follow the idea of \citet{bai2022nearoptimal} and define a \textit{balanced transition} kernel $p^\star$ on $\cX$. First, we denote by $\Ataus{x}=\sum_{x'\geq x} \As{x'}$ the total number of actions that follow information set $x\in\cX$ in the whole sub-tree. The transition kernel $p^\star$ is then defined such as, at each step, the probability of transitioning to the next suitable information set $x$ is proportional to this value. More precisely, $p^\star_0(x_1)=\Ataus{x_1}/\AX$ for all roots $x_1\in\cX_1$, and for any $h\in[H-1]$,
\[p^\star_h(x_{h+1} | x_h,a_h):=\frac{\Ataus{x^{}_{h+1}}}{\sum_{x'_{h+1}\in\cX_{h+1}(x_h,a_h)}\Ataus{x'_{h+1}}}\CommaBin\]
where $\cX_{h+1}(x_h,a_h)$ denotes all the possible information sets of depth $h+1$ that follow $(x_h,a_h)$.

This definition differs from the balanced transition ${p^\star}^h$ in the appendix of \citet{bai2022nearoptimal} for two reasons: 
\begin{itemize}[ itemsep=-2pt,leftmargin=6pt]
\item It is defined globally for all depths $h$, using directly the whole tree instead of a truncated version up to each depth.
\item It uses the reachable number of action instead of the reachable number of states, as the number of action at each set may not be constant.
\end{itemize}

\paragraph{Policy update \eqref{eqn:U1}} 

For any transition kernel $p^\nu$ on $\cX$ and max-player policy~$\mu$, the vector $p^\nu_{1:h}.\mu_{1:h}$ defined by $p^\nu_{1:h}\cdot\mu_{1:h}^{}(x_h,a_h):=p^\nu_{1:h}(x_h)\mu_{1:h}^{}(x_h,a_h)$ is a probability distribution over the information set-action $\cAXh$ at depth $h$.  Along the lines of \FTRL algorithms, we define the policy update of \BalancedFTRL as 
{\small
\begin{align*}
\label{eqn:U1}
\mu^t =\textrm{argmin}&_{\mu\in \maxpi} \scal{\mu_{1:}}{\tilde{L}^{t-1}} + \sum_{h=1}^H\Psi_h\pa{\,p^\star_{1:h}\,\cdot\,\mu^{}_{1:h}\,}/{\eta_h} \qquad
\tag{U1}    
\end{align*}}
where $(\Psi_h)_{h\in [H]}$ are regularizers of $\cAXh$ and $(\eta_h)_{h\in[H]}$ is a sequence of learning rates. While the ideal choice for regret minimization would be  the actual adversarial transitions $p^{\nu^t}$ (replacing our $p^\star$) these transitions are unknown by the player. The balanced transitions instead give acceptable bounds against any adversarial transition.  \BalancedFTRL implements this idea, using either Tsallis entropy \citep{tsallis1987possible} or Shannon entropy as regularizer. Similarly, the IX parameters $\gamma_h$ are also adjusted for each state with $\gamma^\star_h(x_h)=\gamma_h/p^\star_{1:h}(x_h)$.

\begin{algorithm}[t]
\caption{\BalancedFTRL-\textcolor{dgreen}{Tsallis}/\textcolor{dblue}{Shannon}}
\label{alg:ftrl_bal}
\begin{algorithmic}[1]
			\STATE \textbf{Input:}\\
			~~~~ Tree-like structure of $\cAX$\\\vspace{.025cm}
			~~~~ Sequences $(\eta_h)_{h\in [H]}$ of learning rates \\\vspace{.05cm}
                ~~~~ Sequences $(\gamma_h)_{h\in [H]}$ of IX parameters \\\vspace{.05cm}
			~~~~ Regularizers for $w_h\in\Delta(\cAXh):$\\
			~~~~~~ \textcolor{dgreen}{Tsallis: {\scriptsize$\displaystyle
			\Psi_h(w_h)=-\!\sum_{(x_h,a_h)}w_h(x_h,a_h)^\tsallis$}}\\
			~~~~~~ \textcolor{dblue}{Shannon: 
			{\scriptsize $\displaystyle\Psi_h(w_h)=\!\sum_{(x_h,a_h)} w_h(x_h,a_h)\log(w_h(x_h,a_h))$}}
			\STATE \textbf{Initialize:}\\
			~~~~ \textbf{For} all depth $h$ and $x_h\in \cX_h$ :\\
			~~~~~~~~ $\gamma^\star_h(x_h)\gets \gamma_h/p_{1:h}^\star(x_h)$\\\vspace{.025cm}
			\STATE \textbf{For}\;$t=1$ to $T$\\
			~~~~ \textbf{Compute} update~\eqref{eqn:U1}:\\
			{\scriptsize$\displaystyle \mu^t \gets \textrm{argmin}_{\mu\in \maxpi} \scal{\mu_{1:}}{\tilde{L}^{t-1}} +\sum_{h=1}^H\Psi_h\pa{\,p^\star_{1:h}\,.\,\mu^{}_{1:h}\,}/{\eta_h} $}\\ 
			~~~~ \textbf{For} $h=1$ to $H$\vspace{.05cm}\\
			~~~~~~~~ \textbf{Observe} information set $x^t_h$\\\vspace{.025cm}
			~~~~~~~~ \textbf{Execute} $a_h^t \sim \mu^t_h(.|x_h^t)$ and \textbf{receive} reward $r_h^t$\\\vspace{.05cm}
			~~~~~~~~ \textbf{Compute} loss:\\
			~~~~~~~~ $\tell_h^t(x^t_h,a^t_h) \gets (1-r_h^t)/\big(\mu^t_{1:h}(x_h^t,a_h^t)+\gamma^\star_h(x^t_h)\big)$\\\vspace{.1cm}
\end{algorithmic}
\end{algorithm}	

\paragraph{Time complexity} 
\BalancedFTRL first needs to initially compute the balanced transitions, which can be done recursively from the bottom of the information set tree with a time complexity of $\cO(\AX)$.

Then, with \BalancedFTRLTsallis, each update~\eqref{eqn:U1} can be done by solving  a convex programming over $\R^{\AX}$. While this implies a computation with a polynomial cost in~$\AX$, it is still quite expensive when $\AX$ gets large. On the other hand, we will see in Section~\ref{par:equiv} that the update~\eqref{eqn:U1} of \BalancedFTRLShannon has an equivalent formulation with the dilated Shannon entropy. Using Algorithm~\ref{al:adapt_update} of Appendix~\ref{app:algorithmic_update}, each update can then be computed with a time complexity of $\cO(HA)$, where $A$ is an upperbound on the local number of actions any information set can have.

\paragraph{Theoretical analysis}

We now state the main theoretical results for \BalancedFTRL with proof in Appendix~\ref{appendix:bal}.

\begin{restatable}{theorem}{thmbal}\label{thm:hp_bal}
    Let $\delta\in (0,1)$ and $\iota=\log(3\AX/\delta)$. Fix the IX parameters to $\gamma_h=\sqrt{H\iota/(2\AX T)}$ for all $h\in[H]$.
    
    Using \BalancedFTRLTsallis, $\tsallis\in(0,1)$ and constant learning rates $\eta_h=\sqrt{2\tsallis(1-\tsallis)/T}\pa{H/\AX}^{\tsallis-1/2}$, with probability at least $1-\delta$,
    \[\regret_\mathrm{max}^T\leq \pa{\sqrt{2/(\tsallis(1-\tsallis))}+3\sqrt{2\iota}} \sqrt{H\AX T}\,.\]  
    Using \BalancedFTRLShannon and constant learning rates $\eta_h=\sqrt{{2H\log(\AX)}/\pa{\AX T}}$, with probability at least $1-\delta$,
    \[\regret_\mathrm{max}^T\leq \pa{\sqrt{2\log(\AX)}+3\sqrt{2\iota}}\sqrt{H\AX T}\,.\]
\end{restatable}
\begin{remark}
The best theoretical rate is obtained with Tsallis entropy with $\tsallis=1/2$, similarly to \citet{abernethy2015fighting}. Even though Shannon entropy has an extra $\sqrt{\log(\AX)}$ factor in the (expected) first term of the bound as in the multi-armed bandits setting, both entropies give a rate $\cO(\sqrt{H\log(\AX/\delta)\AX T})$ with high-probability.  Indeed, the expected term is dominated in both cases by the $\iota$ term from the high-probability analysis.
\end{remark}
Using Theorem~\ref{thm:folklore}, we conclude that the lower bound $\tcO(H(\AX+\BY)/\epsilon^2)$ on the sample complexity is matched up to logarithmic factors when both players use \BalancedFTRL. In particular, \BalancedFTRL improves by a factor $H$ over the rate of \BalancedOMD by \citet{bai2022nearoptimal}, see Table~\ref{tab:sample_complexity}.

\begin{corollary}\label{cor:complexity}
    For any $\epsilon>0$ and $\delta\in(0,1)$, if both player run either \BalancedFTRL-Tsallis/-Shannon with the parameters specified in Theorem~\ref{thm:hp_bal}, $(\overline{\mu},\overline{\nu})$ is an $\epsilon$-optimal profile with probability at least $1-\delta$ as long as
    \[T\geq \cO\pa{H\pa{\AX+\BY}\log(1/\delta+\AX+\BY)/\epsilon^2}.\]
\end{corollary}
The average profile can be computed along the update of the current profile, see \citet{kozuno2021learning}, at the price of a final update with a $\cO(\AX+\BY)$ time complexity.

\paragraph{Fixed action set size} 
When the size of the action set is fixed, the number of information sets having a given $x\in\cX$ in its history will increase at least exponentially as the depth increases. This can be exploited by \BalancedFTRL with suitable learning rates and IX parameters that both decrease exponentially with the depth $h$ in order to remove the $H$ dependency in the upper bounds, again nearly matching the lower bounds $\tcO((\AX+\BY)/\epsilon^2)$ given by Theorem~\ref{thm:lb} for this setting. The results, proved in Appendix~\ref{appendix:bal}, are stated for Shannon entropy: 

\begin{restatable}{theorem}{thmfixed} Assume that the size of the action set is constant equal to $A$ and let $\delta\in (0,1)$, $\iota=\log(3\AX/\delta).$ Then, \BalancedFTRLShannon with $\gamma_h=\sqrt{A^{H-h}\iota/(2\AX T)}$ and $\eta_h=\sqrt{{2A^{H-h}\log(\AX)}/\pa{\AX T}}$ yields
     \[\regret_\mathrm{max}^T\leq \pa{5\sqrt{\log(\AX)}+8\iota} \sqrt{\AX T}.\]
As in Corollary \ref{cor:complexity}, this allows to compute an $\epsilon$-optimal profile as long as long as 
    \[T\geq \cO\pa{\pa{\AX+\BY}\log(1/\delta+\AX+\BY)/\epsilon^2}.\]
\end{restatable}

%% file: main/practical_algorithm.tex
\section{Near-optimal rate without the knowledge of the tree structure}

In this section we drop the assumption on the knowledge of the  tree structure of the information states.

\paragraph{Dilated entropy and policy update \eqref{eqn:U2}} \label{prg:dilated}

For a given vector $\eta^t$ of learning rates on $\cX$, we define the weighted dilated Shannon divergence \citep{kroer2015faster} between two policies $\mu^1, \mu^0 \in \maxpi$ by 
{\small
\[D_{\eta^t}(\mu^1,\mu^0):=\sum_{h=1}^H\sum_{(x_h,a_h)\in\cAXh} \frac{\mu_{1:h}^1(x_h,a_h)}{\eta^t_h(x_h)} \log\frac{\mu^1_h(a_h|x_h)}{\mu^0_h(a_h|x_h)}.\]}

We are now interested in the updates
\begin{equation}\label{eqn:U2}
\mu^t=\textrm{argmin}_{\mu\in \maxpi} \scal{\mu_{1:}}{\tilde{L}^{t-1}} +\mathcal{D}_{\eta^t}(\mu,\mu^0)
\tag{U2}
\end{equation}

for $\mu^0\in\maxpi$ a base policy fixed over all iterations. As shown by Algorithm~\ref{al:adapt_update} of Appendix~\ref{app:algorithmic_update}, these updates enjoy an efficient computation with a $\cO(HA)$ time complexity if the vectors $\eta^{t-1}$ of learning rates only change at the visited information sets $(x^{t-1}_1,\ldots ,\,x^{t-1}_H)$.


\paragraph{Connections with the balanced algorithm} \label{par:equiv} While there is no general equivalence between the Shannon entropy and the dilated Shannon entropy, the updates \eqref{eqn:U1} of \BalancedFTRLShannon may be put in the form of updates \eqref{eqn:U2}. Precisely, using the fact that $p^\star$ is a transition kernel over $\cX$, the updates~\eqref{eqn:U1} of \BalancedFTRL can be rewritten for all $t\in[T]$ as (see Proposition \ref{prop:equiv_u12})
\[\mu^t =\textrm{argmin}_{\mu\in \maxpi} \scal{\mu}{\tilde{L}^{t-1}}+ D_{\eta^\star}\pa{\mu,\mu^\star},\]
for some $\mu^\star\in \maxpi$ and vector $\eta^\star$ of learning rates computed with a $\cO(\AX)$ time complexity.

\paragraph{Adversarial transitions estimation}

The balanced adversarial transitions $p^\star_{1:h}(x_h)$ used in the above learning rates $\eta^\star_h(x_h)$ are impossible to compute if the tree structure is unknown at the start of the game. In analogy to the loss estimation, we could instead use the estimates of the actual adversarial transitions $p^{\nu^t}_h(x_h)$. For any $(x_h,a_h)\in\cAXh$, an unbiased estimate of $p_h^{\nu^t}(x_h)$ is indeed given by $\hp_{1:h}^{\,t}(x_h,a_h):=\indic{x_h=x_h^t,a_h=a_h^t}/\mu_{1:h}^t(x_h,a_h)$.

Once more, we shall use IX estimators to get bounds with high probability and not just in expectation. However, \BalancedFTRL used the IX parameter $\gamma^\star_h(x_h)=\gamma/p^\star_{1:h}(x_h)$ which also can not be computed for the same reason. Our idea is to again replace these transitions by the same estimated transitions we are currently trying to define. This leads to the following recursive definitions,
\begin{align*}
    \gamma_h^t(x_h,a_h)&:=\gamma/\big(1+\tP_{1:h}^{t-1}(x_h,a_h)\big),\\
    \tp_h^{\,t}(x_h,a_h)&:=\frac{\indic{x_h=x_h^t,a_h=a_h^t}}{\mu_{1:h}^t(x_h,a_h)+\gamma_h^t(x_h,a_h)}\CommaBin
\end{align*}

where $\tP^t:=\sum_{k=1}^t\tp^{\,k}$ are the cumulative estimated transitions and $\gamma_h^t(x_h,a_h)$ the IX parameter that now also depends on the action $a_h$ and on the time $t$.

\paragraph{Adaptive learning rates}
To replace the learning rate $\eta^\star_h(x_h)$ associated to each information set $x_h\in\cX_h$, we thus use the cumulative estimated transitions. As at each round~$t$, for each action $a_h\in\cA(x_h)$, the quantity $\tP^{t}_{1:h}(x_h,a_h)$ is an estimator of $P^t_{1:h}h(x_h)$, we define the estimator $\tP_{1:h}^t(x_h)$, as the average of these estimations
\[\tP^{t}_{1:h}\pa{x_{h}}:=\sum_{a_{h}\in\cA(x_h)}\tP_{1:h}^{t}(x_{h},a^{}_{h})/\As{x_h}\,.\]
This estimator is then used to define a learning rate $\eta^t_h(x_h)$, mimicking the action of $p^\star_h(x_h)$ in $\eta^\star_h(x_h)$,
\[\eta^t_h(x_h) := \min_{x'_{h'}\geq x_h} \eta/\big(1+\tP_{1:h'}^{t-1}(x'_{h'})\big)\,.\]
For technical reasons, these learning rates have to be increasing along a trajectory, hence the $\min$ operator. Note that this definition at round $t$ does not depend on whether the yet unexplored information states are considered in this $\min$ operator, as these verify $\tP^{t-1}_{1:h'}(x'_{h'})=0$.

\paragraph{Adaptive algorithm}

\AdaptiveFTRL is based on update~\eqref{eqn:U2} with the above vectors $(\eta^t)_{t\in [T]}$ of learning rates and IX estimators, which allow it is naturally adaptive to the initially unknown tree structure; see Algorithm~\ref{alg:ftrl_ada} for a detailed description of \AdaptiveFTRL.

\begin{algorithm}[t]
\caption{\AdaptiveFTRL}
\label{alg:ftrl_ada}
\begin{algorithmic}[1]
			\STATE \textbf{Input:}\\
			~~~~ Learning rate $\eta$ and IX base parameter $\gamma$\\
			~~~~ Base policy $\mu^0$ (uniform by default)\\
			\STATE \textbf{For}\;$t=1$ to $T$ :\\
			~~~~ \textbf{For} all depth $h$ and $x_h\in\cX_h$:\\
			~~~~~~~~ $\eta^t_h(x_h)\gets \min_{x'_{h'}\geq x^{}_h} \eta/\big(1+\tP_{1:h'}^{t-1}(x'_{h'})\big)$ \\\vspace{.05cm}
			~~~~~~~~ \textbf{For} all $a_h\in\cA(x_h)$:\\
			~~~~~~~~~~~~ $\gamma^t_h(x_h,a_h)\gets \gamma/\big(1+\tP^{t-1}_{1:h}(x_h,a_h)\big)$ \\\vspace{.1cm}
			~~~~ \textbf{Compute} update~\eqref{eqn:U2}:\\
			~~~~~~~~ $\mu^t \gets \textrm{argmin}_{\mu\in \maxpi} \scal{\mu_{1:}}{\tilde{L}^{t-1}} +\mathcal{D}_{\eta^t}(\mu,\mu^0)$\\
			~~~~ \textbf{For} $h=1$ to $H$:\vspace{.05cm}\\
			~~~~~~~~ \textbf{Observe} information set $x^t_h$\vspace{.05cm}\\
			~~~~~~~~ \textbf{Execute} $a_h^t \sim \mu^t_h(.|x_h^t)$ and \textbf{receive} reward $r_h^t$\vspace{.05cm}\\
			~~~~~~~~ {\small $\tell_h^t(x^t_h,a^t_h) \gets (1-r_h^t)/(\mu^t_{1:h}(x_h^t,a_h^t)+\gamma^t_h(x^t_h,a_h^t))$}\\
			~~~~~~~~ {\small $\tp_{1:h}^t(x^t_h,a^t_h) \gets 1/(\mu^t_{1:h}(x_h^t,a_h^t)+\gamma^t_h(x^t_h,a_h^t))$}
\end{algorithmic}
\end{algorithm}

\paragraph{Time complexity}

Unlike the \BalancedFTRL algorithm, there is no initialization cost. Algorithm \ref{al:adapt_update}, in Appendix~\ref{app:algorithmic_update}, gives an efficient way of computing each update~\eqref{eqn:U2} with a $\cO(HA)$ time complexity. Moreover, as~$\tP$ only increases on visited information states, the vectors $\eta^t$ and $\gamma^t$ can be updated in-place with the same $\cO(HA)$ time complexity.

\paragraph{Theoretical results}

We show high-probability sample complexity bounds for \AdaptiveFTRL, but with an additional $H\log(T)\log(1/\delta+T)$ factor on the sample complexity compared to the optimal rate. This result is proved in Appendix~\ref{app:adaptative_FTRL_proof}.

\begin{restatable}{theorem}{adatheorem}\label{thm:main_ada}
    Let $\iota'=\log({3\AX T}/{\delta})$ and $\logtwo=1+\log_2(1+T)$. For any $\delta\in (0,1)$, using \AdaptiveFTRL with $\eta=2\sqrt{\iota' T/(\logtwo\AX)}$, $\gamma=\sqrt{2\iota' HT/(\logtwo\AX)}$ and $\mu^0$ the uniform policy yields with a probability at least $1-\delta$, 
    \[\regret_\mathrm{max}^T\leq 6H\sqrt{\iota'\logtwo\AX T}\,.\]
    As in Corollary \ref{cor:complexity}, this leads to the computation of an $\epsilon$-optimal profile as long as 
    \[T\geq \cO\pa{H^2(\AX+\BY)\log(1/\delta+T)\log(T)/\varepsilon^2}\,.\]
\end{restatable}
The rate of \AdaptiveFTRL improves the one of \IXOMD--- up till now the algorithm with the best rate without the knowledge of the structure---by a factor at least $\min(X,Y)$.

\paragraph{Parameter tuning} \AdaptiveFTRL needs the knowledge of the $\AX$ constant but only in order to tune the learning rate and the IX parameter. Fortunately, the doubling trick can be applied to $\AX$ similarly to how it would be applied to $T$, which means that, up to an additional constant factor, this knowledge is not required for the theoretical guarantees. 

\label{sec:practical_algorithm}

%% file: main/experiments.tex
\section{Experiments}

In this section we provide preliminary experiments for \BalancedFTRL and \AdaptiveFTRL on simple games. The code used to generate these experiments is available at \url{https://github.com/anon17893/IIG-tree-adaptation}.

\begin{figure*}[ht]
\centering
\includegraphics[width=0.9\textwidth]{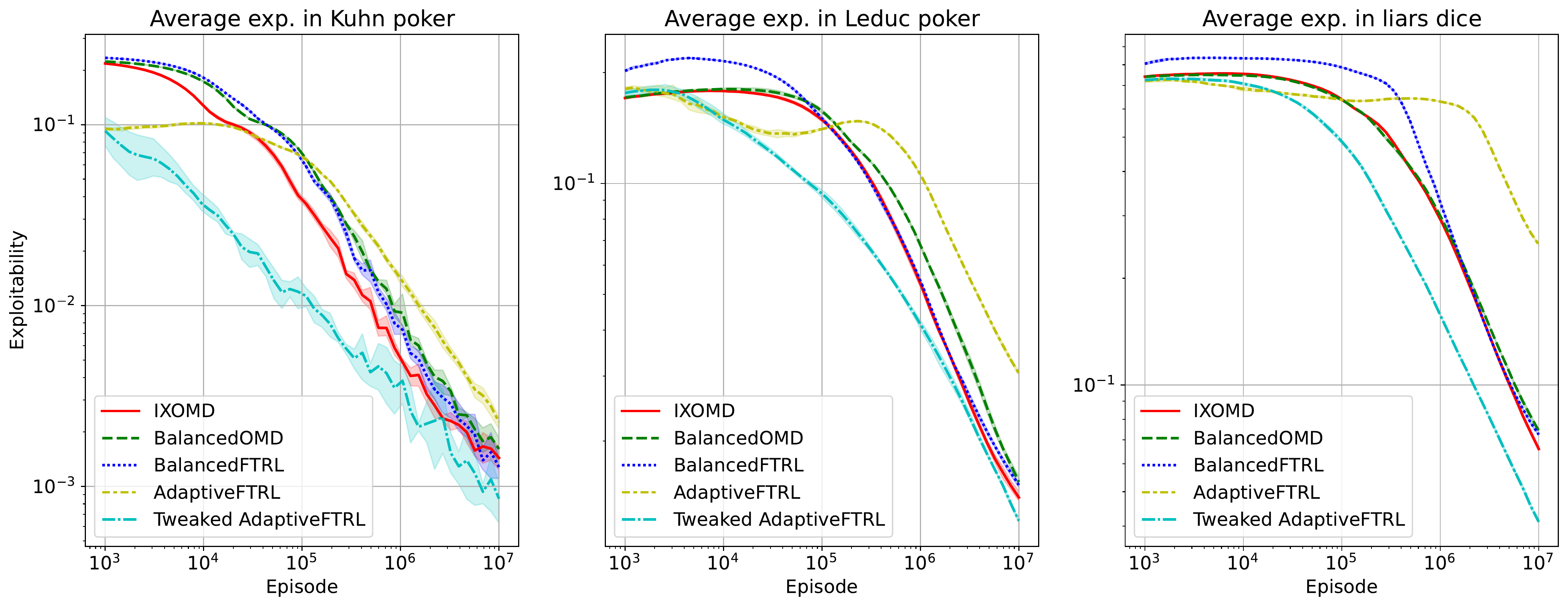}
\caption{Performances of various algorithms with respect to the number of episodes for the time-averaged profile of Theorem~\ref{thm:folklore}. The results are given in terms of the exploitability gap $\max_{\mu\in\maxpi} V^{\mu,\overline{\nu}}-\min_{\nu\in\minpi} V^{\overline{\mu},\nu}$, with the rewards scaled between $0$ and $1$. The total number of actions is always the same for both players, with $\AX=12$ for Kuhn poker, $\AX=1092$ for Leduc poker and $\AX=24570$ for Liars dice.}
\label{sec:experiments}
\end{figure*}

\paragraph{Games} We
compare the algorithms on the following three standard benchmark tabular games: Kuhn poker \citep{kuhn1950extensive}, Leduc poker \citep{southey2005bayes} and liars dice. We use the implementation of these game available in the OpenSpiel library \citep{openspiel}.

\paragraph{Implementation details} We implemented \IXOMD \citep{kozuno2021learning}, \BalancedOMD \citep{bai2022nearoptimal}, \BalancedFTRLShannon and \AdaptiveFTRL. We also implemented \TweakedFTRL, a slight modification of \AdaptiveFTRL, in which the learning rates and IX parameters decrease depending on $(1+\tP^t)^{-1/2}$ instead of $\sqrt{T}/(1+\tP^t)$. This version gets better practical performances, at the price of not having tight theoretical guarantees. 

These algorithms strongly rely on the learning rates for their performances. However, the learning rates obtained with the theoretical analysis of the above algorithms are only adjusted with respect to the worst possible $2$-player games of a given size. With such rates, the practical performances are actually barely improved compared to the theoretical bounds. We thus tune the rates separately for each algorithm and each game, using a (logarithmic) grid search on the global learning rates, while the base IX parameter was taken as $1/20$ of this global learning rate.

\paragraph{Results} Figure~\ref{sec:experiments} shows the results of these experiments. We notice that, despite having very different theoretical guarantees, all algorithms seem to have very similar practical performances once tuned, except for \AdaptiveFTRL being a little worse and its modified version \TweakedFTRL being slightly better on average. We believe that the behaviour of \AdaptiveFTRL can be explained by its learning rates of order $\sqrt{T}/(1+\tP^t)$, which is very large at the beginning leading to a fast convergence in the early phase. But the learning rates then need some time to decrease enough, as $\tP^t$ has to grow, and allow \AdaptiveFTRL to make more progress 
 \footnote{Note that this effect is responsible for an additional $\sqrt{\log(T)}$ factor in the regret analysis}. \TweakedFTRL stays adaptive but avoids this issue with a less steep learning rate decrease and a smaller initial learning rate.



%% file: main/conclusion.tex
\section{Conclusion}
\label{sec:conclusion}

We presented two algorithms. \BalancedFTRL is problem-independent optimal for learning an $\epsilon$-optimal profile. It requires to know the structure of the information set spaces in advance in order to compute a balanced transition kernel used to weight the regularizer. \AdaptiveFTRL does not require the knowledge of the information set space structure while being near problem-independent optimal. Indeed, \AdaptiveFTRL directly estimates the transition kernels induced by the opponent instead of relying on a balanced transition kernel.

Our results bring the following interesting future directions:

\paragraph{Optimal rate for \AdaptiveFTRL} Is it possible to save the extra $H$ dependence over the horizon in the sample complexity of \AdaptiveFTRL?

\paragraph{Last iterate convergence} Can we obtain theoretical guarantees for the current final profile instead of the average profile used in Theorem~\ref{thm:folklore}, in the same way as \citet{lee2021last}? Indeed, the computation of the time-averaged profile is not straightforward outside of the current tabular setting \citep{heinrich2015fictitious,brown2019deep}.

\paragraph{Avoiding importance sampling} In large games, the importance exploration estimates tend to have a very high variance. This may be an issue if we want to combine the algorithms with function approximation \citep{mcaleer2022escher}. Therefore, an interesting future direction would be to design algorithms that \textit{do not rely on neither the importance sampling nor the knowledge of the information set structure, but still get sharp sample complexity guarantees}.

%% file: appendix/table_notations.tex
\section{Notation}
\label{app:notations}

\begin{table}[h]
	\centering
	
	\begin{tabular}{@{}l|l@{}}
		\toprule
		\textbf{Notation} & \textbf{Meaning} \\ \midrule
	$\cS$ & state space of size $S$\\
    $p_h$ & state-transition probability kernel at step $h$\\
    $r_h$ & reward function at step $h$\\
 	$H$ & length of one episode\\
    $\cX$ &information set of the max-player of size $X$\\
    $\cY$ &information set of the min-player of size $Y$\\
	$\cAx$ & action space at information set $x$ of size $\Ax$\\
 	$\cBy$ & action space at information set $y$ of size $\By$\\
    $\cAX = \cup_{h} \cAXh$ & union of information set-action space $\cAXh$ at step $h$ for the max-player of size $\AX$.\\
    $\cBY = \cup_{h} \cBYh$ & union of information set-action space $\cBYh$ at step $h$ for the max-player of size $\BY$.\\
    $\cX_{h+1}(x_h,a_h)$ & information sets of depth $h+1$ that follow $(x_h,a_h)$\\
	\hline
    $\mu_h(\cdot|x_h)$ & policy of the max-player at information-set $x_h$ and step $h$ \\
    $\nu_h(\cdot|y_h)$ & policy of the min-player at information-set $y_h$ and step $h$ \\
    $\mu^{}_{1:h}(x_h,a_h)=\prod_{h'=1}^h \mu_{h'}(a_{h'}|x_{h'})$ & realization plan of the max-player\\
    $p_{h}^{\nu}(x_{h+1}|x_{h},a_{h})$& probability to transition to $x_{h+1}$ from $(x_{h},a_{h})$ when the opponent is $\nu$\\
    $r_h^\nu(x_h,a_h)$ & average reward at $(x_h,a_h)$ when the opponent is $\nu$\\
    $p_{1:h}^\nu(x_h)=p^{}_0(x_1) \prod_{h'=1}^{h-1} p_{h'}^{\nu}(x_{h'+1}|x_{h'},a_{h'})$ & adversarial realization plan against $\nu$\\
    $\ell^{\nu}_h(x_h,a_h)= p^\nu_{1:h}(x_h)\pa{1-r_h^\nu(x_h,a_h)}$ & adversarial loss against $\nu$\\
    \hline
    $T$ & number of episodes\\
    $\regret^T_{\mathrm{max}}$ & regret of max-player\\
    $\ell^t=\ell^{\nu^t}$ & loss of the max-player at episode $t$\\
    $\gamma_h$ & IX parameter\\
    $\tell_h^{\,t}$ & IX estimate of the loss of the max-player\\
   $\tL^t=\sum_{k=1}^t\tell^{k}$ & cumulative loss  of the max-player\\
   $p^\nu_{1:h}.\mu_{1:h}=p^\nu_{1:h}(x_h)\mu_{1:h}^{}(x_h,a_h)$ & induced probability distribution over $\cAXh$ at depth $h$\\ 
   \hline
   $p^\star_h$ & balanced transition at step $h$\\
   $\Psi_h$ & regularizers on $\cAXh$ \\
   $\gamma^\star_h(x_h)=\gamma_h/p^\star_{1:h}(x_h)$ & adjusted IX parameter at information set $x_h$\\
   \hline
   $D_{\eta^t}(\mu_1,\mu_0)$ & weighted (by $\eta^t$) dilated Shannon divergence on $\cX$\\       $\gamma_h^t(x_h,a_h)$ & adaptive IX parameter at $(x_h,a_h)$ \\
    $\tp_h^{\,t}(x_h,a_h)$ & estimated transition at $(x_h,a_h)$\\
    $\tP^t=\sum_{k=1}^t\tp^{\,k}$ & cumulative estimated transitions\\
    $\eta^t_h(x_h)$ & adaptive learning rate at $x_h$\\
    \bottomrule
	\end{tabular}
 \caption{Table of notation use throughout the paper}
\end{table}

%% file: appendix/related_work.tex
\section{Extended related work}
\label{app:related_work}
In this section we review previous work on learning an $\epsilon$-optimal strategies in IIGs.

\paragraph{Full feedback} When the game is known, that is the information set structure space, transitions probability and reward function are provided a first line of work recasts the setting through the sequence-form representation of a game as a linear program which can be
solved efficiently \citep{Rom62,VONSTENGEL1996220,koller1996efficient}. A second line of work relies on first-order optimization method for saddle point computation \citep{hoda2010smoothing, kroer2015faster,kroer2018solving,kroer2020faster,munos2020fast,lee2021last}.
In particular \citet{hoda2010smoothing,kroer2018solving} relies on the Nesterov smoothing technique \citet{nesterov2005smooth} whereas \citet{kroer2015faster,kroer2020faster} use the \MirrorProx algorithm \citep{nemirovski2004prox}. These methods has a rate of convergence of order $\tcO(\poly(H,\AX,\BY)/\epsilon)$. Note that the rate is smaller than the lower bound of Theorem~\ref{thm:lb} since they assume full knowledge of the game. Even game dependent exponential rate can be obtained with these type of approach, see \citet{gilpin2012first} and \citet{munos2020fast}.

A third approach, counterfactual regret minimization \citep{zinkevich2007regret}, leverages local regret minimisation, i.e. minimising a type of regret at each information set. Popular algorithm are based on the regret-matching algorithm \citep{hart2000simple,gordon2007no} such as \CFR algorithm \citep{zinkevich2007regret} or based on a close variant of regret-matching, e.g. \CFRp \citep{tammelin2014solving, burch2019revisiting,farina2020faster}. Note that other local regret minimizers could be used, see for example \citet{ waugh2014unified,farina2019regret}. These algorithms enjoy a guarantee of convergence of order $\tcO(\poly(H,\AX,\BY)/\epsilon^2)$. 

Nevertheless, all the methods described above need to \emph{the information set tree} (or all the state space) or in order to compute one update. The cost of one traversal is of order $\cO(X+Y)$ if the transitions and the actions of the other player are sampled; see for example the
external-sampling \MCCFR algorithm \citep{lanctot2009monte}.

\paragraph{Trajectory feedback} A way to tackle the aforementioned issues is to consider the agnostic setting where the \emph{agent
has no prior knowledge of the game and only observes trajectories of the game}. Precisely, the rewards, the transition probabilities are unknown. This is the setting considered in this paper.

\paragraph{Model-based} A first method to deal with this limited feedback is to build a \emph{model} of the game then run any full feedback algorithm in this model. For example, \citet{zhou2020posterior} use \textit{posterior sampling} (PS, \citealp{strens2000bayesian}) to learn a model and then use the \CFR algorithm in games sampled from the posterior. They obtain a convergence rate of order $\tcO(\poly(H,S,A,B)/\epsilon^2)$ but only when the games are actually sampled according to the known prior. Instead, \citet{zhang2020finding} rely on the principle of optimism in presence of uncertainty to incrementally build a model of the game. Then, the \CFR algorithm is fed with \emph{optimistic estimates} of the local regrets. They prove a high-probability sample complexity of order $\tcO(\poly(H,S,A,B)/\epsilon^2)$.

\paragraph{Model-free} Another line of work \citep{lanctot2009monte,johanson2012efficient,shcmid2018variance,farina2020stochastic} directly estimate the local regret via importance sampling that are then feed to the \CFR algorithm. In particular, the outcome-sampling \MCCFR \citep{lanctot2009monte, farina2020stochastic} builds an importance sampling estimate of the counterfactual regret by playing according to a well-chosen \emph{balanced policy}. Intuitively, this policy should ensure to \emph{explore all the information sets}. Note that, depending on the structure of the information set space, playing uniformly over the actions at each information set is not necessarily a good choice. Instead, \citet{farina2020stochastic} propose as a balanced policy to play action with probability proportional to the number of leaves in the sub-tree of possible next information sets. In particular, the outcome-sampling \MCCFR algorithm requires the knowledge of the information set space structure to build its balanced policy. Nonetheless, in order to obtain $\epsilon$-optimal strategies with high probability, \MCCFR needs at most
 $\tcO(H^3(\AX+\BY )/\epsilon^2)$ realizations of the game \citep{farina2020stochastic,bai2022nearoptimal}.

Latter, \citet{kozuno2021learning} propose to combine \emph{Online Mirror Descent (\OMD)} with \emph{dilated Shannon entropy as regularizer} and importance sampling estimate of the losses of a player, see also \citet{farina2021bandit}. They prove a sample complexity, for the  proposed algorithm, \IXOMD, of order $\tcO(H^2(X\AX+Y\BY  )/\epsilon^2)$. Interestingly, they do not need to know in advance the structure of the information set space to obtain this bound\footnote{They only need to know the size of the information set spaces to tune optimally the learning rate}. However, the sample complexity of \IXOMD does not match the lower bound for this setting which is of order $\cO((\AX+\BY)/\epsilon^2)$, see Section~\ref{sec:lower_bound} and \citet{bai2022nearoptimal}. Recently, \citet{bai2022nearoptimal} propose the \BalancedOMD algorithm that also relies on \OMD but with a dilated entropy weighted by the realization plans of balanced policies as regularizers. These balanced policies generalize the one introduced by \citet{farina2020stochastic} for all depths in the information set tree. For \BalancedOMD, they prove a sample complexity of order $\tcO(H^3(\AX+\BY)/\epsilon^2)$ with an improved dependency on the sizes of the information set spaces at the price of a worse dependence in the horizon $H$.

\paragraph{Perfect information Markov game} Another line of work consider Markov game \citet{kuhn1953extensive} with \emph{perfect} information and limited feedback. However, they do not assume perfect recall. \citet{sidford2020soving,zhang2020model,daskalakis2020independent,wei2021last} consider the case where a \emph{generative model} is available whereas \citet{wei2017online,bai2020near,xie2020learning,liu2021sharp} deal with the \emph{trajectory feedback} case. Although this setting is related to ours there is no direct comparison between the two.

%% file: appendix/lower_bound_tadashi.tex
\section{Lower bound proof details}
\label{app:lower_bound}

To formally state lower bounds, we start with a tad of additional notation.
For the proofs of the lower bounds, we consider games in which one player (say the min-player) has no effect on both rewards and state-transition dynamics.
Therefore, we omit random variables related to the min-player.

An algorithm is a sequence $\cL \df (\cL^t)_{t=1}^{T+1}$, where $\cL^t$ is a measurable mapping from $\cAX^{(t-1)H} \times \R^{(t-1)H}$ to $\maxpi$.
In $t$-th episode, the algorithm outputs a policy $\mu^t = \cL^t (H_t)$ given the history $H_t \df (x_h^u, a_h^u, r_h^u)_{u=1, h=1}^{t-1, H}$ and follows it.
By Ionescu–Tulcea theorem \citep[Theorem 3.3]{lattimore2020bandit_book}, there exists a probability space consistent with this procedure.
After $T$-th episode, the algorithm outputs a final policy $\mu^{T+1} = \cL^{T+1} (H_{T+1})$.

We prove the following theorems. Note that we assume deterministic rewards, but lower bounds show that the difficulty of IIGs is the same as the difficulty of bandit problems when $H=1$. This is because we made use of the partial observability such that rewards look stochastic from the view point of players.

\begin{theorem}[Regret Lower Bound; Variable Action Space Size]
    Fix horizon $H$, number of episodes $T$, total number of actions $\AX$, and a positive scalar $\delta \in (0, 1)$.
    If $\AX < H$, there is no such game.
    If $\AX \geq H$, $\delta < 1/4$, and $T \geq - 0.4 K \log (4\delta)$, for any algorithm, there exists a game such that with probability greater than $\delta$, the algorithm suffers from the regret of
    \begin{align*}
        \Omega \pa{
            \sqrt{(\AX - H) \min \{\AX - H, H\} T \log \left(1/(4\delta)\right)}
        }\,.
    \end{align*}
\end{theorem}

\begin{theorem}[Sample Complexity Lower Bound; Variable Action Space Size]
    Fix horizon $H$, total numbers of actions $\AX,\, \BY$, and positive scalars $\delta \in (0, 1),\, \varepsilon \in (0, H)$.
    If either $\AX < H$ or $\BY < H$, there is no such game.
    If $\AX \geq H$, $\BY \geq H$, and $\delta < 1/4$, for any algorithm, there exists a game such that the algorithm needs at least
    \begin{align*}
        \Omega \pa{
            \dfrac{
                (\AX - H) \min \{\AX - H, H\} + (\BY - H) \min \{\BY - H, H\}
            }{
                \varepsilon^2
            } \log \left(1/(4\delta)\right)
        }
    \end{align*}
    episodes to output $\varepsilon$-NE with probability at least $1 - \delta$.
\end{theorem}

\begin{theorem}[Regret Lower Bound; Fixed Action Space Size]
    Fix horizon $H$, number of episodes $T$, number of actions $A$ at each information set, number of information sets $X$, and a positive scalar $\delta \in (0, 1)$.
    If $X < (A^H - 1) / (A-1)$, there is no such game.
    If $X < 2 (A^H - 1) / (A-1)$, $\delta < 1/4$, and $T \geq - 0.4 K \log (4\delta)$, for any algorithm, there exists a game such that with probability greater than $\delta$, the algorithm suffers from the regret of
    \begin{align*}
        \Omega \pa{
            \sqrt{ T \AX \log \left(1/(4\delta)\right)}
        }\,.
    \end{align*}
\end{theorem}

\begin{theorem}[Sample Complexity Lower Bound; Fixed Action Space Size]
    Fix horizon $H$, number of actions $A,\, B$ at each information set, numbers of information sets $X,\, Y$, and positive scalars $\delta \in (0, 1),\, \varepsilon \in (0, H)$.
    If either $X < (A^H - 1) / (A-1)$ or $B < (B^H - 1) / (B-1)$, there is no such game.
    If $(A^H - 1) / (A-1) \leq X < 2 (A^H - 1) / (A-1)$, $(B^H - 1) / (B-1) \leq Y < (B^H - 1) / (B-1)$, and $\delta < 1/4$, for any algorithm, there exists a game such that the algorithm needs at least
    \begin{align*}
        \Omega \pa{
            \dfrac{
                 \AX + \BY
            }{
                \varepsilon^2
            } \log \left(1/(4\delta)\right)
        }
    \end{align*}
    episodes to output $\varepsilon$-NE with probability at least $1 - \delta$.
\end{theorem}

\subsection{Variable action space sizes}

We first consider a case in which the number of viable actions depends on information sets.

\paragraph{Proof sketch} In the hard game instance we  consider, we make actions of one player ineffective depending on which one of $\AX$ and $\BY$ is bigger.
Specifically, if $\AX > \BY$, the min-player's actions have no effect on the state-transition dynamics and rewards.\footnote{In case of ties, make actions of the min-player ineffective.}
Then the game reduces to a Markov decision process, and finding an $\varepsilon$-optimal policy is equivalent to finding an $\varepsilon$-NE.
For the time being, we assume that $\AX > \BY$.

Due to the perfect-recall assumption (at least one successive information set for each action except the final time step), $\AX$ must be greater than or equal to $H$.
Furthermore, when $\AX = H$, there is only one policy, and the sample complexity lower bound is trivially $0 = \AX - H$.
Hereafter we assume that $\AX > H$.

\begin{figure}[tbh!]
    \centering
    \includegraphics[width=0.9\textwidth]{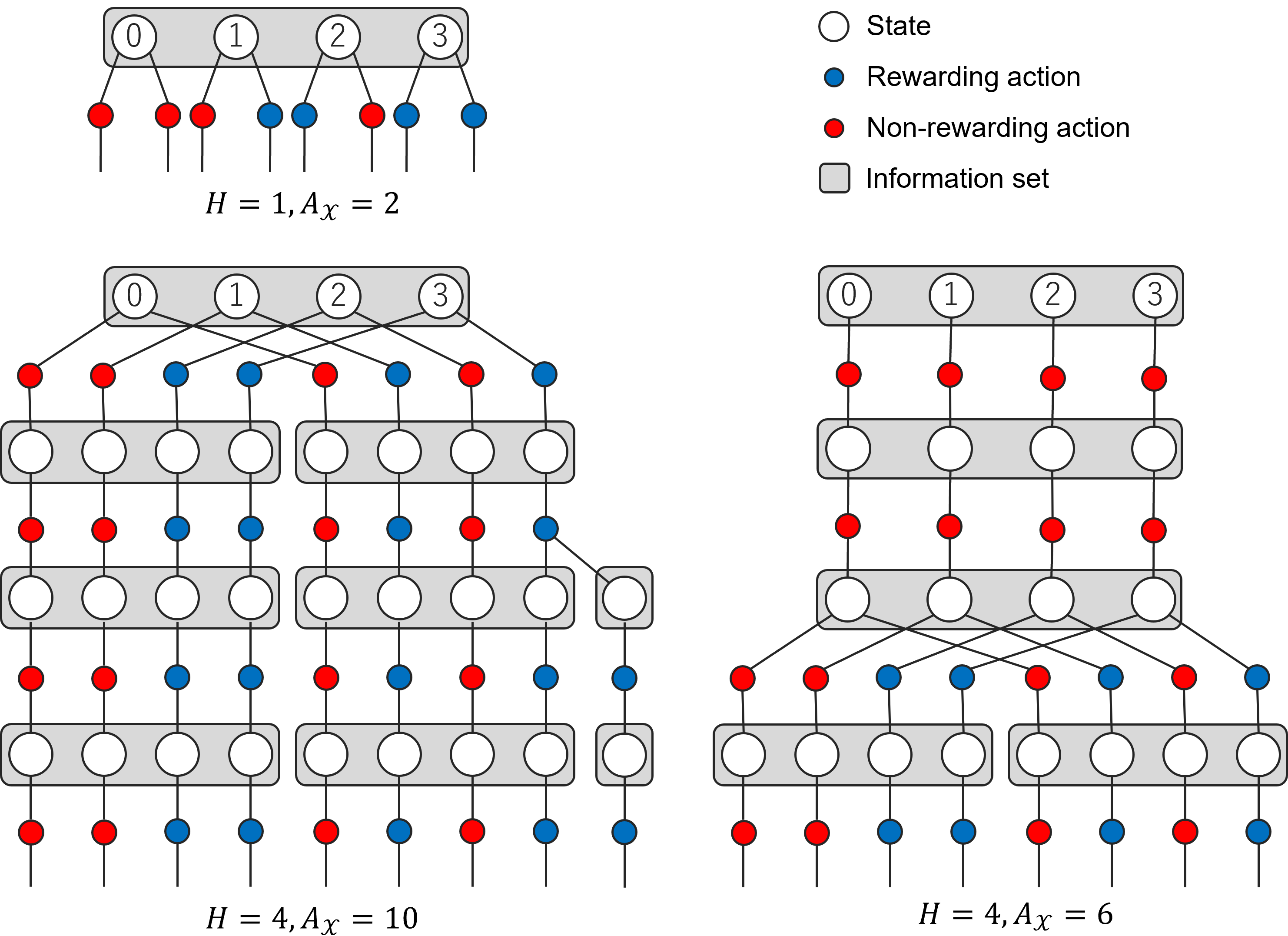}
    \caption{A hard game instance with information-set-dependent action space from the view point of the max-player.}
    \label{fig:hard instance with variable A}
\end{figure}

Figure~\ref{fig:hard instance with variable A} depicts hard game instances under different conditions.
White circles are states, and states at the top (and identified by integers) are initial states.
Edges from each state are viable actions at a corresponding state.
Blue circles are rewarding actions, where the max-player receives the reward of $1$.
Red circles are non-rewarding actions.
Shaded boxes are information sets, and all states in a box belong to the same information set.

When $\AX \geq 2H$, we construct a game shown in the bottom left panel.
For brevity let $K \df \floor{\AX / H} \geq \AX / H - 1$ and assume that $\AX$ is divisible by $H$.
In this case there is no right-most branch.
If $\AX$ is not divisible by $H$, append a small branch as shown in the figure and set state-transition probability to the small branch $0$.
In this game, there are $2^K$ initial states, at which there are $K$ viable actions.
Each initial state is identified by an integer from $0$ to $2^K-1$.
An initial state is sampled as follows:
in the beginning of each episode, $K$ binary integers are sampled from $K$ (ordered) Bernoulli distributions and concatenated to construct a binary representation of an integer from $0$ to $2^K-1$ (e.g., $010111$), which is an initial state in this episode.
The parameter of $k$-th Bernoulli distribution is denoted by $p_k$.
The reward value for $k$-th action is the $k$-th bit of the initial state in binary representation.
After the first step, the state transitions in a deterministic way as shown in the figure, and reward values at later time steps are the same as the reward value at the first time step.

To get an intuition about this hard game instance, it is beneficial to consider a case where $H=1$.
In this case, the game essentially reduces to a bandit problem shown in the top left panel of Figure~\ref{fig:hard instance with variable A}.
In this case the regret lower bound of $\Omega ( \sqrt{K T} )$ and the sample complexity lower bound of $\Omega ( K / \varepsilon^2 )$ are known to hold.
When the reward is scaled by $H$, the lower bounds become $\Omega ( \sqrt{H (\AX - H) T} )$ and $\Omega ( H (\AX - H) / \varepsilon^2 )$, respectively.
The hard game instance described above instantiates this idea.

The hard game construction described above does not work when $H > 1$ and $2 H > \AX > H$.
In this case, we construct a game shown in the bottom right panel.
The idea is almost the same, but there are only $4$ initial states, and
the player makes decision at $(2H - \AX + 1)$-th time step instead of the first time step.
There the max-player has only $2$ actions, and reward values at $(2H - \AX + 1)$-th time step and thereafter are determined as before.
This game is essentially a $2$-arm Bernoulli bandit with reward scaled by $\AX - H$.
In this case the regret lower bound of $\Omega ( (\AX - H) \sqrt{2 T} )$ and the sample complexity lower bound of $\Omega ( 2 (\AX - H)^2 / \varepsilon^2 )$ are known to hold.

If $\AX > \BY$ does not hold, we can derive the same result with $\AX$ replaced by $\BY$ by exchanging the role of max- and min-players in the previous argument.
Consequently, depending on which one of $\AX$ and $\BY$ is bigger, we can use $\AX$ or $\BY$ in lower bounds.
The claimed results follow from a fact that $\max \{ a , b \} \geq (a+b) / 2$.

\paragraph{Full Proof.} We assume that $\AX \geq 2H$ as other cases can be handled similarly.
We consider $\floor{\AX / H} + 1$ games and identify each game by an integer from $0$ to $\floor{\AX / H}$.
In $0$-th game , we set all $p_k$ to $0.5$.
In $k$-th game, we set $(p_j)_{j \neq k}$ to $0.5$, and $p_k$ to $0.5 + \Delta$, where $\Delta \in (0, 0.5)$ is specified later in the proof.
In other words, an optimal arm in $k$-th game is $k$-th action.
We let $\P_k$ denote the probability measure induced by the interconnection of the algorithm and game $k$.
Expectation under $\P_k$ is denoted by $\E_k$.
We also let $\P_k'$ and $\E_k'$ denote the law of $(x_h^u, a_h^u, r_h^u)_{u=1, h=1}^{T, H}$ and expectation under $\P_k'$ when the algorithm is run in game $k$.

The proofs rely on the following core lemma.

\begin{lemma}\label{lemma:probability lower bound lemma}
    Fix scalars $\rho \in (0, 1)$ and $v \in (0, \infty)$.
    For brevity let $K \df \floor{\AX / H}$.
    If $\Delta \leq 0.3$, for any measurable function $f: \cAX^{TH} \times \{0, 1\}^{TH} \to \{0, 1\}$ and $k \in [K]$, it holds that
    \begin{align*}
        \P_k \pa{ f (H_{T+1}) = 1 }
        >
        \exp \pa{ - s - v } \pa{
            \P_0 \pa{ f (H_{T+1}) = 1 } - \rho - \frac{4 \Delta^2}{v} n_k
        }\,,
    \end{align*}
    where $s \df \dfrac{4 \Delta}{3} \log \dfrac{1}{\rho} + \sqrt{2 v \log \dfrac{1}{\rho}}\,$ and $n_k \df \E_0 \bra{ \sum_{t=1}^T \indic{a_1^t=k} }$.
\end{lemma}

\begin{proof}
    Let $E$ be an event $\{ f(H_{T+1}) = 1 \}$.
    We begin with rewriting $\P_k \pa{ E }$ as follows:
    \begin{align}\label{eq:intermediate lower bound for P_theta}
        \P_k \pa{ E }
        =
        \E_k \bra{ f (H_{T+1}) }
        =
        \E_k' \bra{ f (H_{T+1}) }
        =
        \E_0' \bra{ f (H_{T+1}) \frac{d \P_k'}{d \P_0'} (H_{T+1}) }\,.
    \end{align}
    We derive a lower bound for $d \P_k' / d \P_0' (H_{T+1})$.
    Because $\cAX^{TH} \times \{0, 1\}^{TH}$ is discrete,
    $d \P_k' / d \P_0' (H_{T+1})$ is simply an importance sampling ratio.
    Furthermore, changing the game from $0$ to $k$ only changes the reward distribution when $k$-th action is taken.
    Therefore,
    \begin{align*}
        \frac{d \P_k'}{d \P_0'} (H_{T+1})
        &=
        \prod_{t=1}^T \bra{
            r_1^t \pa{ \frac{0.5 + \indic{a_1^t=k} \Delta}{0.5} }
            +
            (1 - r_1^t) \pa{ \frac{0.5 - \indic{a_1^t=k} \Delta}{0.5} }
        }
        \\
        &=
        \prod_{t=1}^T \bra{
            r_1^t \pa{ 1 + 2 \indic{a_1^t=k} \Delta }
            +
            (1 - r_1^t) \pa{ 1 - 2 \indic{a_1^t=k} \Delta }
        }
        =
        \prod_{t=1}^T \pa{1 + P_t}\,,
    \end{align*}
    where
    $
        P_t
        \df
        2 \Delta \pa{
            2 r_1^t - 1
        } \indic{a_1^t=k}
    $.
    From a fact that
    $
        1 + x \geq \exp \pa{ x - x^2 }
    $
    for $x \geq - 0.6$, we deduce that
    \begin{align*}
        \frac{d \P_k'}{d \P_0'} (H_{T+1})
        \geq
        \exp \pa{ - S_T - V_T }\,,
    \end{align*}
    where we defined $S_T \df - \sum_{t=1}^T P_t$ and $V_T \df \sum_{t=1}^T P_t^2 = 4 \Delta^2 \sum_{t=1}^T \indic{a_1^t=k}$.
    
    Returning to equation (\ref{eq:intermediate lower bound for P_theta}),
    it clearly holds that
    \begin{align*}
        \P_k \pa{ E }
        &\geq
        \E_0' \bra{ f(H_{T+1}) \indic{S_T < s, V_T \leq v } \frac{d \P_k' }{d \P_0'} (H_{T+1}) }
        \\
        &>
        \exp \pa{ - s - v } \E_0' \bra{ f(H_{T+1}) \indic{S_T < s, V_T \leq v } }
        \\
        &= \exp \pa{ - s - v } \E_0 \bra{ f(H_{T+1}) \indic{S_T < s, V_T \leq v } }\,.
    \end{align*}
    We lower-bound $f(H_{T+1}) \indic{S_T < s, V_T \leq v }$.
    To this end, note that $A \cap B = (A \cup B^c) \cap B$ for any sets $A$ and $B$. Therefore,
    \begin{align*}
        f(H_{T+1}) \indic{S_T < s, V_T \leq v }
        &=
        f(H_{T+1}) \pa{
            1 - \indic{S_T \geq s, V_T \leq v } - \indic{V_T > v }
        }
        \\
        &\geq
        f(H_{T+1}) - \indic{S_T \geq s, V_T \leq v } - \indic{V_T > v }\,,
    \end{align*}
    and thus,
    \begin{align*}
        \E_0 \bra{ f(H_{T+1}) \indic{S_T < s, V_T \leq v } }
        \geq
        \P_0 \pa{E}
        - \underbrace{
            \P_0 \pa{S_T \geq s, V_T \leq v }
        }_{
            \leq \rho \text{ by (a)}
        }
        - \underbrace{
            \P_0 \pa{V_T > v}
        }_{
            \leq \E_0[V_T] / v \text{ by (b)}
        }\,.
    \end{align*}
    It is easy to see that (b) follows from Markov's inequality.
    To see why (a) follows, let $\cF_t'$ be a $\sigma$-algebra generated by $(s_h^u, a_h^u)_{u=1, h=1}^{t-1, H}$ and $a_1^t$.
    Then, (a) follows from Freedman's inequality by noting that $(P_t)_{t=1}^n$ is a martingale difference sequence with respect to the filtration $(\cF_t')_{t=1}^T$, $P_t$ is almost surely bounded by $2 \Delta$, and $\sum_{t=1}^T \E_0 [ (P_t - \E_0 [ P_t | \cF_t'] )^2 | \cF_t] = V_T$ as $\E_0 [ P_t | \cF_t'] = 0$.
    Consequently,
    \begin{align*}
        \P_k (E)
        >
        \exp \pa{ - s - v } \pa{
            \P_0 \pa{E} - \rho - \frac{4 \Delta^2}{v} \E_0 \bra{ \sum_{t=1}^T \indic{a_1^t=k} }
        }
        =
        \exp \pa{ - s - v } \pa{
            \P_0 \pa{E} - \rho - \frac{4 \Delta^2}{v} n_k
        }\,.
    \end{align*}
    This concludes the proof.
\end{proof}

\paragraph{Regret Lower Bound.}
Now we are ready to prove the regret lower bound.
For each $k \in [K]$, let $E_k$ be an event $\sum_{t=1}^T \mu^t_1 (a_k | x_1) \leq 0.5 T$.
Note that while implicit, $\sum_{t=1}^T \mu^t_1 (a_k^* | x_1^t)$ is a function of $H_{T+1}$, and thus,
$E_k$ can be rewritten in a form of $f (H_{T+1}) = 1$.
As we verify later, our choice of $\Delta$ is smaller than $0.3$.
Therefore from the previous lemma,
\begin{align*}
    \max_{k \in [K]} \P_k (E_k)
    \geq
    \frac{1}{K} \sum_{k \in [K]} \P_k (E_k)
    >
    \exp \pa{ - s - v } \pa{
        \frac{1}{K} \sum_{k = 1}^K \P_0 \pa{E_k} - \rho - \frac{4 T \Delta^2}{K v}
    }\,.
\end{align*}
Note that $\sum_{k \in [K]} \P_0 \pa{E_k} = \E_0 \left[ \sum_{k \in [K]} \indic{E_k} \right] \geq K-1$ since $\indic{E_k} = 0$ for some $k$ implies that $\indic{E_j} = 1$ for $j \neq k$. Thus,
\begin{align*}
    \max_{k \in [K]} \P_k (E_k)
    >
    \exp \pa{ - s - v } \pa{
        1 - \frac{1}{K} - \rho - \frac{4 T \Delta^2}{K v}
    }\,.
\end{align*}
Setting $v = \dfrac{32 T \Delta^2}{K}$ and $\rho = \dfrac{1}{8}$,
\begin{align*}
    \max_{k \in [K]} \P_k (E_k)
    &>
    \frac{1}{4} \exp \pa{ - 4 \Delta \log 2 - 8 \Delta \sqrt{\dfrac{3T}{K} \log 2} - \dfrac{32 T \Delta^2}{K} }\,,
\end{align*}
where we used an assumption that $K \geq 2$.
Equating the right hand side to $\delta$ and solving for $\Delta$, we choose
\begin{align}
    \Delta
    =
    \sqrt{
        \pa{ \frac{K}{16T} \log 2 + \frac{1}{8} \sqrt{\frac{3K}{T} \log 2} }^2
        + \frac{K}{32T} \log \frac{1}{4 \delta}
    } - \frac{K}{16T} \log 2 - \frac{1}{8} \sqrt{\frac{3K}{T} \log 2}
    <
    \sqrt{
        \frac{K}{32T} \log \frac{1}{4 \delta}
    } < 0.3\,,\label{eq:Delta value}
\end{align}
where the middle inequality follows since $\sqrt{a + b} < \sqrt{a} + \sqrt{b}$ for $a, b \in (0, \infty)$.
Then it holds that $\max_{k \in [K]} \P_k (E_k) > \delta$.
Because the event $E_k$ implies that the regret is greater than or equal to $0.5 T \Delta H$ in $k$-th game, we have a lower bound $\Omega (\sqrt{ - H^2 TK \log (4\delta) }) = \Omega (\sqrt{ - H^2 T (\AX / H - 1) \log (4\delta) }) = \Omega (\sqrt{ - H T (\AX - H) \log (4\delta) })$.

\paragraph{Sample Complexity Lower Bound.}
Now we are ready to prove the sample complexity lower bound.
Let $E_k$ be an event $\mu^{T+1}_1 (a_k | x_1) \leq 0.5$ for each $k \in [K]$.
By the same argument as the one in the regret lower bound proof, $\max_{k \in [K]} \P_k (E_k) > \delta$ for $\Delta$ in Equation~\ref{eq:Delta value}.
Since in $k$-th game, the event $E_k$ implies that the simple regret of $\mu^{T+1}$ is greater than or equal to $0.5 \Delta H$, we have a lower bound for the simple regret of $\Omega ( \sqrt{ - H (\AX - H) / T \log (4\delta) } )$.
Therefore unless $T \geq - H (\AX - H) / \varepsilon^2 \log (4\delta)$,
the algorithm fails to output $\varepsilon$-optimal policy with probability greater than $\delta$.

\subsection{Fixed action space size}

Now we turn to the case where the action space depends on information sets. Since the proof is relatively easy, we directly give a full proof.

In the hard game instance we will consider, we make actions of one player ineffective depending on which one of $\AX$ and $\BY$ is bigger, as before.
For the time being, we assume that $\AX > \BY$.
Furthermore we restrict rewards to be binary but allow them to be stochastic.
As we did in the proof of Lemma~\ref{lemma:probability lower bound lemma}, we can create a game in which rewards are deterministic but look stochastic from the view point of players.

First of all, we claim that there is no game such that $X < (A^H - 1) / (A-1)$. This is because $X \geq 1 + A + \cdots + A^{H-1} = (A^H - 1) / (A-1)$ due to the perfect-recall assumption (at least one successive information set for each action except the final time step). Hereafter we assume that $X \geq (A^H - 1) / (A-1)$.

\begin{figure}[tbh!]
    \centering
    \includegraphics[width=0.9\textwidth]{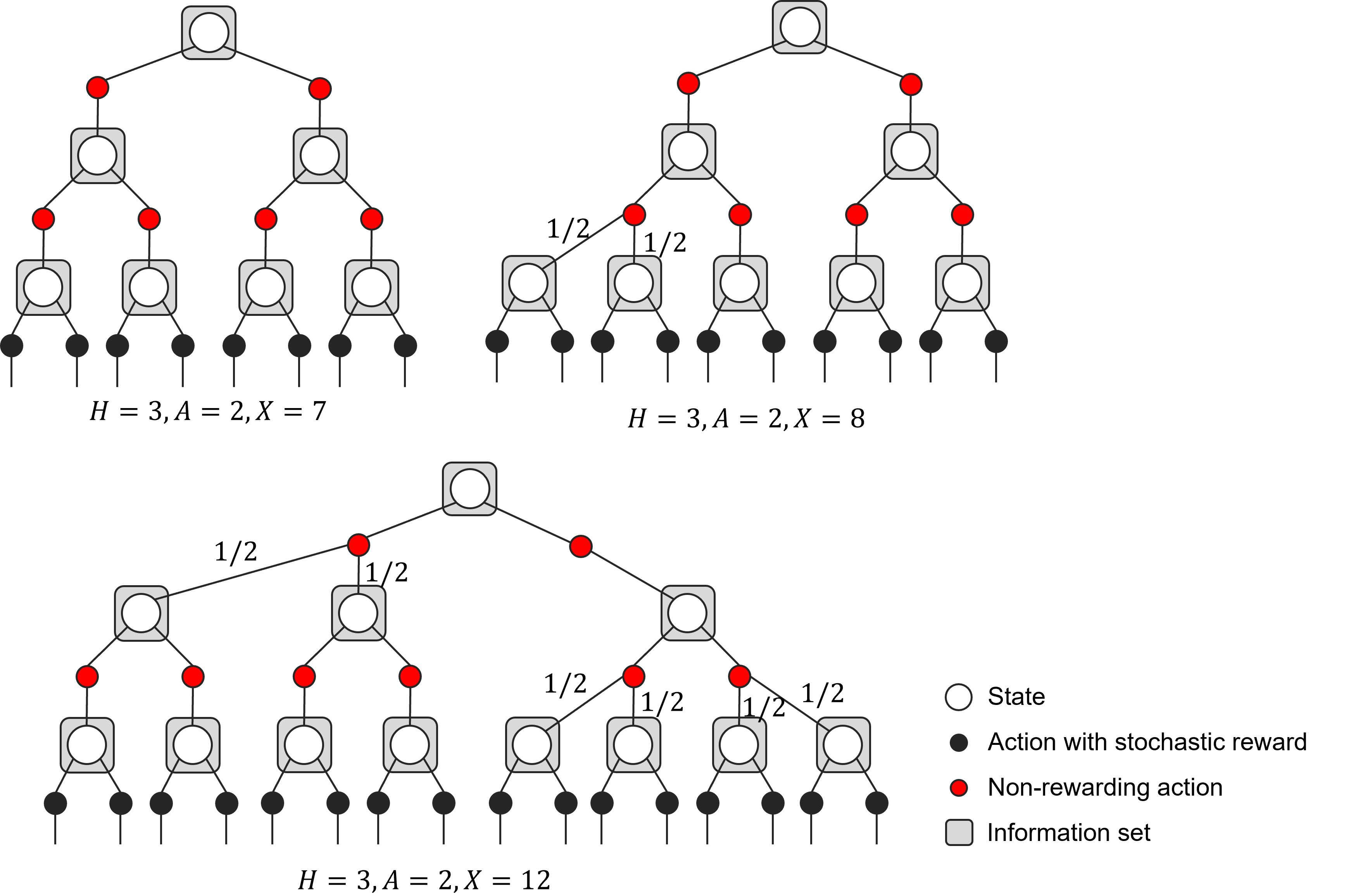}
    \caption{A hard game instance with information-set-independent action space from the view point of the max-player.}
    \label{fig:hard instance with fixed A}
\end{figure}

Figure~\ref{fig:hard instance with fixed A} depicts the hard game instance under different conditions.
In all cases, all information sets are singleton. Unless a number is given, all edges from action nodes to state nodes mean deterministic transitions.
Numbers along edges indicate state-transition probabilities of the edges.
In all cases, the max-player receives reward $1$ at leaf action nodes with some probability specified later in the proof.

The top left panel shows the case where $X = (A^H - 1) / (A-1)$. In this case, the game tree is $A$-ary tree with deterministic transitions. This case is dealt in \citet{bai2022nearoptimal}, and we obtain a similar (pseudo-)regret bound.

The top right panel shows the case where $X = (A^H - 1) / (A-1) + 1$.
In this case, an additional $A$-ary tree with depth $1$ is appended at the left-most action node at $(H-1)$-th time step.
State-transition probabilities to the successor state nodes are $0.5$.
Up to $(A^H - 1) / (A-1) + A^{H-1}$, additional $A$-ary trees with depth $1$ are added similarly to different action nodes at $(H-1)$-th time step one by one from left to right action nodes.
For example, if $X = (A^H - 1) / (A-1) + 2$, another additional $A$-ary tree with depth $1$ is appended at the second left-most action node at $(H-1)$-th time step.

The bottom panel shows the case where $X = (A^H - 1) / (A-1) + A^{H-1} + 1$.
In this case, we remove the additional $A$-ary tree at the left-most action node at $(H-1)$-th time step and append an $A$-ary tree with depth $2$ at the left-most action node at $(H-2)$-th time step.
Up to $X = (A^H - 1) / (A-1) + A^{H-1} + A^{H-2}$, additional $A$-ary trees with depth $1$ are added similarly to different action nodes at $(H-2)$-th time step one by one from left to right action nodes.
This procedure can be repeated until all action nodes at the first time step have additional branches. At this point, there are $(A^H - 1) / (A-1) + A^{H-1} + \cdots + A = 2 (A^H - 1) / (A-1) - 1$ information sets.

Now we specify reward probabilities at leaf action nodes and construct $A^H + A \min \{ X - (A^H - 1) / (A-1), A^{H-1} \} + 1$ games.
We index leaf action nodes by integers from $1$ to $2 A^H$.
We let $a_i^*$ and $x_i^*$ denote the corresponding action of $i$-th node and its predecessor information set (i.e., an information set at which $a_i^*$ can be taken).
In $0$-th game , we set all reward probabilities to $0.5$.
In $k$-th game, we set reward probabilities of all leaf action nodes except $k$-th one to $0.5$, and the reward probability of $k$-th leaf action node to $0.5 + \Delta$, where $\Delta \in (0, 0.5)$ is specified later in the proof.
In other words, an optimal leaf action node in $k$-th game is $k$-th one.
We let $\P_k$ denote the probability measure induced by the interconnection of the algorithm and game $k$.
Expectation under $\P_k$ is denoted by $\E_k$.
We also let $\P_k'$ and $\E_k'$ denote the law of $(x_h^u, a_h^u, r_h^u)_{u=1, h=1}^{T, H}$ and expectation under $\P_k'$ when the algorithm is run in game $k$.

By using the almost same proof as that of Lemma~\ref{lemma:probability lower bound lemma}, we can deduce the following lemma.

\begin{lemma}\label{lemma:probability lower bound lemma fixed action version}
    Fix scalars $\rho \in (0, 1)$ and $v \in (0, \infty)$.
    For brevity let $K \df 2 A^H$.
    If $\Delta \leq 0.3$, for any measurable function $f: \cAX^{TH} \times \{0, 1\}^{TH} \to \{0, 1\}$ and $k \in [K]$, it holds that
    \begin{align*}
        \P_k \pa{ f (H_{T+1}) = 1 }
        >
        \exp \pa{ - s - v } \pa{
            \P_0 \pa{ f (H_{T+1}) = 1 } - \rho - \frac{4 \Delta^2}{v} n_k
        }\,,
    \end{align*}
    where $s \df \dfrac{4 \Delta}{3} \log \dfrac{1}{\rho} + \sqrt{2 v \log \dfrac{1}{\rho}}\,$ and $n_k \df \E_0 \bra{ \sum_{t=1}^T \indic{x_H^t=x_k^* a_H^t=a_k^*} }$.
\end{lemma}

\paragraph{Regret Lower Bound.}
Now we are ready to prove the regret lower bound.
For each $k \in [K]$, let $E_k$ be an event $\sum_{t=1}^T \mu^t_{1:H} (x_k^*, a_k^*) \leq 0.5 T$.
Note that while implicit, $\sum_{t=1}^T \mu^t_{1:H} (x_k^*, a_k^*)$ is a function of $H_{T+1}$, and thus,
$E_k$ can be rewritten in a form of $f (H_{T+1}) = 1$.
As we verify later, our choice of $\Delta$ is smaller than $0.3$.
Therefore from the previous lemma,
\begin{align*}
    \max_{k \in [K]} \P_k (E_k)
    \geq
    \frac{1}{K} \sum_{k \in [K]} \P_k (E_k)
    >
    \exp \pa{ - s - v } \pa{
        \frac{1}{K} \sum_{k = 1}^K \P_0 \pa{E_k} - \rho - \frac{4 T \Delta^2}{K v}
    }\,.
\end{align*}
Note that $\sum_{k \in [K]} \P_0 \pa{E_k} = \E_0 [ \sum_{k \in [K]} \indic{E_k} ] \geq K-A^{H-1}$.
To see why it is the case, suppose that $\indic{E_k} = 0$ for some $k$, and let $(x_h^*, a_h^*)_{h=1}^H$ be information sets and actions from the root to $(x_k^*, a_k^*)$.
Then $\sum_{t=1}^T \mu^t_{1:H} (x_k^*, a_k^*) > 0.5 T$ means that $\sum_{t=1}^T \mu^t_{1:h} (x_h^*, a_h^*) > 0.5 T$ for all $h$ since $\mu^t_{1:h} (x_h^*, a_h^*) = \mu^t_{1:h-1} (x_{h-1}^*, a_{h-1}^*) \mu^t_h (a_h^* | x_h^*)$.
Furthermore it also holds that $\sum_{t=1}^T \mu^t_{h} (a_h^* | x_h^*) > 0.5 T$ for any $h$.
If there is no branch along the path to the leaf action node to the root node, every trajectory to leaf action nodes certainly contains $x_1^*$.
Therefore $\sum_{t=1}^T \mu^t_1 (a_1 | x_1^*) \leq 0.5 T$ holds for any $a_1 \in \cA \setminus \{a_1^*\}$, and $\indic{E_k} = 0$ for some $k$ implies that $\indic{E_j} = 1$ for some $j \neq k$.
If there is an additional branch, the same argument holds for all leaf action nodes except those in the branch.
In the worst case, the number of leaf action nodes in a branch is $A^{H-1}$.
Thus,
\begin{align*}
    \max_{k \in [K]} \P_k (E_k)
    >
    \exp \pa{ - s - v } \pa{
        1 - \frac{1}{A} - \rho - \frac{4 T \Delta^2}{K v}
    }\,.
\end{align*}
Setting $v = \dfrac{32 T \Delta^2}{K}$ and $\rho = \dfrac{1}{8}$,
\begin{align*}
    \max_{k \in [K]} \P_k (E_k)
    &>
    \frac{1}{4} \exp \pa{ - 4 \Delta \log 2 - 8 \Delta \sqrt{\dfrac{3T}{K} \log 2} - \dfrac{32 T \Delta^2}{K} }\,,
\end{align*}
where we used an assumption that $A \geq 2$.
Equating the right hand side to $\delta$ and solving for $\Delta$, we choose
\begin{align}
    \Delta
    =
    \sqrt{
        \pa{ \frac{K}{16T} \log 2 + \frac{1}{8} \sqrt{\frac{3K}{T} \log 2} }^2
        + \frac{K}{32T} \log \frac{1}{4 \delta}
    } - \frac{K}{16T} \log 2 - \frac{1}{8} \sqrt{\frac{3K}{T} \log 2}
    <
    \sqrt{
        \frac{K}{32T} \log \frac{1}{4 \delta}
    } < 0.3\,,\label{eq:Delta value 2}
\end{align}
where the middle inequality follows since $\sqrt{a + b} < \sqrt{a} + \sqrt{b}$ for $a, b \in (0, \infty)$.
Then it holds that $\max_{k \in [K]} \P_k (E_k) > \delta$.
Because the event $E_k$ implies that the regret is greater than or equal to $0.25 T \Delta H$ in $k$-th game, we have a lower bound $\Omega (\sqrt{ - T A^H \log (4\delta) }) = \Omega (\sqrt{ - TXA \log (4\delta) })$, where $X \leq 2 (1 + A + \cdots + A^{H-1}) = X_H (1 + 1/A + \cdots + 1/A^{H-1}) \leq 2 X_H$ is used.

\paragraph{Sample Complexity Lower Bound.}
Now we are ready to prove the sample complexity lower bound.
Let $E_k$ be an event $\mu^{T+1}_1 (x_k^*, a_k^*) \leq 0.5$ for each $k \in [K]$.
By the same argument as the one in the regret lower bound proof, $\max_{k \in [K]} \P_k (E_k) > \delta$ for $\Delta$ in Equation~\ref{eq:Delta value 2}.
Since in $k$-th game, the event $E_k$ implies that the simple regret of $\mu^{T+1}$ is greater than or equal to $0.25 \Delta H$, we have a lower bound for the simple regret of $\Omega ( \sqrt{ - XA / T \log (4\delta) } )$.
Therefore unless $T \geq - XA / \varepsilon^2 \log (4\delta)$,
the algorithm fails to output $\varepsilon$-optimal policy with probability greater than $\delta$.

%% file: appendix/appendix_copy.tex
\section{General tools}
\label{app:general_tools}

All proofs of the following sections will focus on the max player. We will denote by $\cAX^t$ its history at episode $t$ and use $\filtration$ the filtration such that $\salgebra_0=\{\emptyset,\Omega\}$, $\salgebra_t=\sigma\pa{\nu^1,r^1,\cAX^1 ... \:\nu^t,r^t,\cAX^t,\nu^{t+1}}$ and $\salgebra_{T}=\sigma\pa{\nu^1,r^1,\cAX^1, ... \:, \nu^{T},r^T \cAX^T}$ is the filtration of $\sigma$-algebra generated by the relevant random variables for the max-player up to the round $t+1$.

We first provide a slight generalization of Lemma~1 by  \citet{neu2015explorea} for our settings, as $\gamma^t$ is not fixed in advance in the adaptive case.

\begin{lemma}\label{lemma:concentration}
Let $h\in[H]$, $(x_h,a_h)\in\cAXh$, $\tau_h(x_h,a_h)$ a stopping time with respect to $\filtration$, and $\gamma'_h(x_h,a_h)>0$ a fixed constant such that for all $t\leq \tau_h(x_h,a_h)$, $\gamma_h^t(x_h,a_h)\geq \gamma'_h(x_h,a_h)$. Then, with probability $1-\delta'$,

\[\tL^{\tau_h(x_h,a_h)}_h(x_h,a_h)-L^{\tau_h(x_h,a_h)}_h(x_h,a_h)\leq \frac{\log(1/\delta')}{2\gamma'_h(x_h,a_h)}\,.\]
\end{lemma}

\begin{proof}

We first define the random process $(S^t)_{t\in[|0,t|]}$ with respect to $\filtration$

\[S^t:=\exp\bra{2\gamma'_h(x_h,a_h)\sum_{k=1}^{t\wedge \tau_h(x_h,a_h)}\pa{\tell^k_h(x_h,a_h)-\ell^k_h(x_h,a_h)}}\]

With the inequality $\frac{z}{1+z/2}\leq \log(1+z)$ for $z\geq 0$, we have for all $t\leq \tau_h(x_h,a_h)$, with $\omu_{1:h}^t:=\mu_{1:h}^t(x_h,a_h)$,

\begin{align*}
    2\gamma'_h(x_h,a_h)\tell^t_h(x_h,a_h)&=2\gamma'_h(x_h,a_h)\frac{1-r_h^t}{\omu_{1:h}^t+\gamma_h^t(x_h,a_h)}\indic{x_h=x_h^t,a_h=a_h^t}\\
    &\leq 2\gamma'_h(x_h,a_h)\frac{1-r_h^t}{\omu_{1:h}^t+\gamma'_h(x_h,a_h)(1-r_h^t)}\indic{x_h=x_h^t,a_h=a_h^t}\\
    &=\frac{2\gamma'_h(x_h,a_h)(1-r_h^t)/\omu_{1:h}^t}{1+\gamma'_h(x_h,a_h)(1-r_h^t)/\omu_{1:h}^t}\indic{x_h=x_h^t,a_h=a_h^t}\\
    &\leq \log(1+2\gamma'_h(x_h,a_h)\hl_h^t(x_h,a_h))\,.
\end{align*}

Which gives, using $1+z\leq e^z$,

\begin{align*}
    \E\bra{\exp\pa{2\gamma'_h(x_h,a_h)\tell^t_h(x_h,a_h)} | \salgebra_{t-1}}&\leq \E\bra{ 1+2\gamma'_h(x_h,a_h)\hl_h^t(x_h,a_h)| \salgebra_{t-1}}\\
    &=1+2\gamma'_h(x_h,a_h) \ell^t_h(x_h,a_h)\\
    &\leq \exp\pa{2\gamma'_h(x_h,a_h) \ell^t_h(x_h,a_h)}\,.
\end{align*}

We thus have that $(S^t)_{t\in[|0,T|]}$ is a super-martingale, and using the Markov inequality we get

\begin{align*}
    \P\pa{\tL_h^{\tau_h(x_h,a_h)}(x_h,a_h)-L_h^{\tau_h(x_h,a_h)}(x_h,a_h)\geq \log(1/\delta')/\pa{2\gamma'_h(x_h,a_h)}}&=\P\pa{S^T\geq 1/\delta'}\\
    &\leq \E\pa{S^T}\delta'\\
    &\leq \delta'\,.
\end{align*}
\end{proof}

We also state a simple consequence of Azuma-Hoeffding inequality that we will use multiple times in the following sections.

\begin{lemma}\label{lemma:azuma}

Let $(u_t)_{t\in[T]}$ be a random process adapted to $(\salgebra_t)_{t\in[T]}$ such that $0\leq u_t\leq H$ for all $t\in [T]$. Then, with a probability at least $1-\delta'$,

\[\sum_{t=1}^T \bra{u_t-\E\pa{u_t | \salgebra_{t-1}}}\leq H\sqrt{2T\log(1/\delta')}\,.\]

The same property can also be shown for $\sum_{t=1}^T \bra{\E\pa{u_t | \salgebra_{t-1}}-u_t}$.
\end{lemma}

\begin{proof}
    We define the martingale $(M_t)_{t\in[|0,T|]}$ adapted to $(\salgebra_t)_{t\in[|0,T|]}$ with
    
    \[M_t:=\sum_{k=1}^t \bra{u_k-\E\pa{u_k | \salgebra_{k-1}}}\,.\]
    
    We have, thanks to the hypothesis, for all $1\leq t\leq T$, that $-H\leq M_t-M_{t-1} \leq H$. Azuma-Hoeffding inequality then allows us to conclude
    
    \[\P\pa{M_T \geq H\sqrt{2T\log(1/\delta')}}\leq \exp\pa{-2\frac{2H^2T\log(1/\delta')}{4H^2T}}=\delta'\,.\]
    
    The proof for $\sum_{t=1}^T \bra{\E\pa{u_t | \salgebra_{t-1}}-u_t}$ is exactly the same.
\end{proof}

Finally, we state a lemma useful for the analysis of \AdaptiveFTRL.

\begin{lemma}\label{lemma_logbound}
    Let $\pa{u_t}_{1\leq t\leq T}$ be a non-negative sequence that verifies for all $t\in [T]$, $u_t\leq 1+U_{t-1}$ where $U_t=\sum_{k=1}^t u_k$. Then
    \[\sum_{t=1}^T \frac{u_t}{1+U_{t-1}}\leq \log_2(1+U_T)\,.\]
\end{lemma}

\begin{proof}
    Let $\lambda_t=\frac{u_t}{1+U_{t-1}}$. Using the concavity of $\log_2$, we get $\lambda_t\leq \log_2(1+\lambda_t)$ as $\lambda_t \in[0,1]$ by assumption. Summing over $t$ then yields
    \[\sum_{t=1}^T \lambda_t\leq \sum_{t=1}^T \log_2(1+\lambda_t)=\sum_{t=1}^T\log_2\pa{\frac{1+U_t}{1+U_{t-1}}}=\log_2(1+U_T)\,.\]
\end{proof}

\section{Proof of \texorpdfstring{\BalancedFTRL}{Balanced-FTRL}} \label{appendix:bal}

\paragraph{Space of realization plans} 

Let $\maxpibis:=\left\{\mu_{1:},\mu\in\maxpi\right\}$ the set of max-player realization plans. We can notice that this space is a restriction of  $\R_{\geq 0}^\AX$ to an affine subspace of $\R^\AX$, defined by the $X$ constraints, for each $x_h\in \cX$, 
\[\sum_{a_h\in\cAs{x_h}}\mu_{1:h}(x_h,a_h)=\mu_{1:h-1}(x_{h-1},a_{h-1})\,,\]

with the convention $\mu_{1:0}(x_0,a_0)=1$. We thus write $\maxpibis= (F+u)\cap\R_{\geq 0}^{\AX}$, with $F$ a linear subspace of $\R^{\AX}$ and $u\in\maxpibis$.

\paragraph{Average profile}\label{par:averaged}

We especially get that the average $\frac{1}{T}\sum_{t=1}^T \mu_{1:}^t$ is also in $\maxpibis$. We define an average realization plan $\overline{\mu}$ of the max-player as one of the $\overline{\mu}\in \maxpi$ that verifies
\[\overline{\mu}_{1:}=\frac{1}{T}\sum_{t=1}^T \mu_{1:}^t\,.\]

Doing the same step for $\overline{\nu}$, an average realization of the min-player, we get the existence of an average profile $(\overline{\mu},\overline{\nu})$. Note that if the average profile is positive over all information states, then it is unique. It has the same properties as the one defined in \citet{kozuno2021learning}.

\paragraph{Global regularizer} We will assume that the $\Psi_h$ functions are supported on $\R_{\geq 0}^{A\pa{\cX_h}}$, and denote by $\Psi:\R^{\AX}_{\geq 0} \to \R$ the function  $\Psi(\mu_{1:})=\sum_{h=1}^H\frac{1}{\eta_h}\Psi_h(\mu^{}_{1:h}\cdot p^\star_{1:h})$. This function is strictly convex and non-positive when using either Tsallis or Shannon entropy as $\Psi_h$ function. Its definition with the coordinates of the realization plan will be important as we will consider its gradient. The associated Bregman divergence is

\[\mathcal{D}_{\Psi}(\mu_{1:}^1,\mu_{1:}^2):=\Psi(\mu_{1:}^1)-\Psi(\mu_{1:}^2)-\scal{\nabla\Psi(\mu_{1:}^2)}{\mu_{1:}^1-\mu_{1:}^2}\,.\]

We first give a lemma that justifies the use of the balanced transitions.

\begin{lemma}
\label{lemma:balanced}
    For any $\nu\in\minpi$, we have
    \[\sum_{h=1}^H\sum_{x_h\in\cX_h}\As{x_h}\frac{p^\nu_{1:h}(x_h)}{p^\star_{1:h}(x_h)}=\AX\,.\]
    Furthermore, when the size of the action set is fixed to $A$, we also have for any $h\in[H]$
    \[\sum_{x_h\in\cX_h}\As{x_h}\frac{p^\nu_{1:h}(x_h)}{p^\star_{1:h}(x_h)}\leq A^{h-H} \AX \,.\]
\end{lemma}

\begin{proof}
    We first show by induction on $h$ that that for all $x_h\in\cX_h$,
    \begin{equation}\label{eqn:rec_bal}
        \sum_{h'=1}^{h-1}\sum_{x_{h'}\in\X_{h'}}\As{x_{h'}}\frac{p^\nu_{1:h'}(x_{h'})}{p^\star_{1:h'}(x_{h'})}+\sum_{x_h\in\X_{h}}\Ataus{x_h}\frac{p^\nu_{1:h}(x_{h})}{p^\star_{1:h}(x_{h})}=\AX\,.
    \end{equation}
    
    For $h=1$, we have
    
    \[\sum_{x_1\in\cX_1}\Ataus{x_1}\frac{p^\nu_{1:1}(x_1)}{p^\star_{1:1}(x_1)}=\AX\sum_{x_1\in\cX_1}p_0(x_1) =\AX\,.\]
    
    Assuming the property holds at step $h$, to obtain the property at step $h+1$ we use
    \begin{align*}
        \sum_{x_{h+1}\in\cX_{h+1}}\Ataus{x_{h+1}}\frac{p_{1:{h+1}}^\nu(x_{h+1})}{p_{1:{h+1}}^\star(x_{h+1})}&=\sum_{(x_h,a_h)\in\cAXh}\sum_{(...,x_h,a_h,x_{h+1})}\Ataus{x_{h+1}}\frac{p_{1:{h+1}}^\nu(x_{h+1})}{p_{1:{h}}^\star(x_{h}) p_h^\star(x_{h+1}|x_h,a_h) }\\
        &=\sum_{(x_h,a_h)\in\cAXh}\sum_{(...,x_h,a_h,x_{h+1})}\frac{p_{1:{h+1}}^\nu(x_{h+1})}{p_{1:{h}}^\star(x_{h})}\sum_{(...,x_h,a_h,x'_{h+1})}\Ataus{x'_{h+1}}\\
        &=\sum_{(x_h,a_h)\in\cAXh}\frac{p_{1:{h}}^\nu(x_{h})}{p_{1:{h}}^\star(x_{h})}\sum_{(...,x_h,a_h,x'_{h+1})}\Ataus{x'_{h+1}}\\
        &=\sum_{x_h\in\cX_h}\frac{p_{1:{h}}^\nu(x_{h})}{p_{1:{h}}^\star(x_{h})}(\Ataus{x_h}-\As{x_h})\,.
    \end{align*}
    
As for any $x_H\in\cX_H$, $\Ataus{x_H}=\As{x_H}$ (as no state follows a final state), we obtain the first equality using \eqref{eqn:rec_bal} with $h=H$.

For the inequality, we can notice that, if the size of the action set is fixed equal to $A$, we have for any $h\in[H-1]$:
\begin{align*}
    \Ataus{x_h}&=\sum_{x'\geq x_h}A(x')\geq \sum_{a_h\in \cA(x_h)}\sum_{(...,x_h,a_h,...x_H)}\As{x_H}\\
    &=A\sum_{a_h\in\cA(x_h)}\textrm{Card}\{x_H\in \cX_H | (...,x_h,a_h,...x_H)\}\geq A\sum_{a_h\in\cA(x_h)}A^{H-h-1}=\As{x_h}A^{H-h}\,.
\end{align*}

The inequality $\Ataus{x_h}\geq \As{x_h}A^{H-h}$ also trivially holds for $h=H$ as explained before. Using again \eqref{eqn:rec_bal}, this time with any $h\in [H]$, we obtain
    \begin{align*}
        \AX&=\sum_{h'=1}^{h-1}\sum_{x_{h'}\in\X_{h'}}\As{x_{h'}}\frac{p^\nu_{1:h'}(x_{h'})}{p^\star_{1:h'}(x_{h'})}+\sum_{x_h\in\X_{h}}\Ataus{x_h}\frac{p^\nu_{1:h}(x_{h})}{p^\star_{1:h}(x_{h})}\\
        &\geq \sum_{x_h\in\X_{h}}\Ataus{x_h}\frac{p^\nu_{1:h}(x_{h})}{p^\star_{1:h}(x_{h})}\\
        &\geq A^{H-h}\sum_{x_h\in\X_{h}}\As{x_h}\frac{p^\nu_{1:h}(x_{h})}{p^\star_{1:h}(x_{h})}\,.
    \end{align*}
    Which concludes the proof.
\end{proof}

\begin{lemma}\label{lemma:fixed_size}
    Assume that the size of the action set is fixed equal to $A\geq 2$, then $H\leq\sqrt{\AX}$ and we have the inequality:
    \[\sum_{h=1}^H A^{(h-H)/2}\leq 2+\sqrt{2}\,.\]
\end{lemma}

\begin{proof}
The first inequality is simply obtained with
    \[\AX\geq \sum_{h=1}^H A^h\geq \sum_{h=1}^H 2^h\geq (2^{H+1}-2)\geq H^2\]
where the first inequality comes from the fact that, at each time step $h$, the player remembers its $h-1$ past actions, and the last comes from the fact that $(2^{H+1}-2)/H^2$ is increasing for $H\geq 3$ and is more than $1$ for $H=1,2,3$.

For the second inequality we use

\[\sum_{h=1}^H A^{(h-H)/2}\leq \sum_{h=1}^H 2^{(h-H)/2}\leq \sum_{h=0}^{+\infty} A^{h/2}\leq 1/(1-1/\sqrt{2})=2+2\sqrt{2}\,.\]
\end{proof}

We now try to bound the regret. We first decompose it, using the same decomposition as in \citet{zimmert2019optimal} for the expected part.

\begin{lemma}\label{lemma:bal_decompos}
    The regret of \BalancedFTRL can be decomposed into
    
    \begin{align*}
        \regret_T&\leq \underbrace{\sum_{t=1}^T\scal{\w_{1:}^t}{\ell^t-\tell^t}}_{\textrm{BIAS I}}+\underbrace{\max_{\mu^\dagger\in\maxpi}\sum_{t=1}^T\scal{\w_{1:}^\dagger}{\tell^t-\ell^t}}_{\textrm{BIAS II}}\\
        &\qquad+\underbrace{\max_{\w\in\maxpi}\bra{-\Psi(\w^{}_{1:})}}_{\textrm{REG}}+\underbrace{\sum_{t=1}^T\mathcal{D}_{\Psi^\star}\pa{\nabla\Psi(\w_{1:}^t)-\tell^t,\nabla\Psi(\w_{1:}^t)}}_{\textrm{VAR}}\,.
    \end{align*}
\end{lemma}

\begin{proof}
    Let $\mu^\dagger\in\maxpi$ be some realization plan. For all $t\leq T$, we decompose the instantaneous regret against $\mu^\dagger$ at step $t$ into
    \begin{align*}
        \scal{\w_{1:}^t-\w_{1:}^\dagger}{\ell^t}&=\scal{\w_{1:}^t}{\ell^t-\tell^t}+\scal{\w^\dagger}{\tell^t-\ell^t}\\
            &+\bra{\Phi\pa{-\tL^{t-1}}-\Phi\pa{-\tL^t}-\scal{\w_{1:}^\dagger}{\tell^t}}\\
            &+\bra{\scal{\w_{1:}^t}{\tell^t}+\Phi\pa{-\tL^t}-\Phi\pa{-\tL^{t-1}}}
    \end{align*}
    
    where $\Phi(y)=\sup_{\w\in\maxpi} \scal{\w_{1:}}{y}-\Psi(\w_{1:})$.
    
    Summing the first two terms over $t$ gives the two $\textrm{BIAS}$ terms. For the $\textrm{REG}$ term, summing over $t$ yields, by telescoping,
    
    \begin{align*}
        \sum_{t=1}^T\bra{\Phi\pa{-\tL^{t-1}}-\Phi\pa{-\tL^t}-\scal{\w_{1:}^\dagger}{\tell^t}}&=\Phi(0)-\Phi(-\tL^t)-\scal{\w_{1:}^\dagger}{\tL^t}\\
        &\leq \max_{\w\in\maxpi}\bra{-\Psi(\w_{1:})}+\Psi(\w_{1:}^\dagger)\\
        &\leq \max_{\w\in\maxpi}\bra{-\Psi(\w_{1:})}
    \end{align*}
    
    as $\mu_{1:}^\dagger\in\maxpibis$ for the first inequality, and $\Psi$ is a non-positive function for the second inequality.
    
    We then define $\Psi^\star$ the convex conjugate of $\Psi$, with $\Psi^\star(y)=\sup_{x\in \R^{\AX}_{\geq 0}}\scal{x}{y}-\Psi(x)$, and, thanks to the decomposition $\maxpibis= (F+u)\cap\R_{\geq 0}^{\AX}$, get

    \begin{align*}
    \scal{\w_{1:}^t}{\tell^t}+&\Phi\pa{-\tL^t}-\Phi\pa{-\tL^{t-1}}\\
    &=\scal{\w_{1:}^t}{\tell^t}+\Phi\pa{\nabla\Psi(\w_{1:}^t)+g_t-\tell^t}-\Phi\pa{\nabla\Psi(\w_{1:}^t)+g_t}\tag{1}\\
    &=\scal{\w_{1:}^t}{\tell_t}+\Phi\pa{\nabla\Psi(\w_{1:}^t)-\tell^t}-\Phi\pa{\nabla\Psi(\w_{1:}^t)}\tag{2}\\
    &\leq\scal{\w_{1:}^t}{\tell_t}+\Psi^*\pa{\nabla\Psi(\w_{1:}^t)-\tell^t}-\Psi^*\pa{\nabla\Psi(\w_{1:}^t)}\tag{3}\\
    &=\mathcal{D}_{\Psi^\star}\pa{\nabla\Psi(\w_{1:}^t)-\tell^t,\nabla\Psi(\w_{1:}^t)}\,.
\end{align*}

Where we used successively:

(1) As $\w^t=\argmin_{\w\in\maxpi} \scal{\w_{1:}}{\tL^{t-1}}+\Psi(\w_{1:})$, we have
\[\tL^{t-1}+\nabla\Psi\pa{\w_{1:}^t}+g^t=0\]
where $g^t\in F^\perp$.

(2) For all $y\in\R^{\AX}$, $\Phi\pa{y+g^t}=\Phi\pa{y}+\scal{u}{g^t}$.

(3) Because $\Phi$ is a constrained version over $\maxpibis$ of $\Psi^*$, $\Phi\leq\Psi^*$.
And as $\Psi$ is convex, $\w_{1:}^t=\textrm{argmax}_{x\in\R_{\geq 0}^{\AX}}\scal{x}{\nabla\Psi(\mu_{1:}^t)}-\Psi(x)$, which implies that the maxima associated to $\Psi^*(\nabla\Psi(\w_{1:}^t))$ is reached on $\maxpibis$.
\end{proof}

We first give upper bounds on the $\textrm{BIAS}$ terms when using the IX estimations.

\begin{lemma}\label{lemma:bal_bias}
For any sequence $\gamma_h=\gamma>0$, with probability at least $1-2\delta/3$, we have

\[\textrm{BIAS I}\leq H\sqrt{2T\iota} + \gamma  \AX T\qquad \textrm{and}\qquad \textrm{BIAS II}\leq H\frac{\iota}{2\gamma}\,.\]

If the size of the action set is fixed and $\gamma_h=A^{(H-h)/2}\gamma$ instead, with probability at least $1-2\delta/3$, we have

\[\textrm{BIAS I}\leq H\sqrt{2T\iota} + (2+\sqrt{2})\gamma  \AX T\qquad \textrm{and}\qquad \textrm{BIAS II}\leq (1+1/\sqrt{2})\frac{\iota}{\gamma}\,.\]

\end{lemma}

\begin{proof}
    We can first decompose $\textrm{BIAS I}$ into
    
    \[\textrm{BIAS I}=\sum_{t=1}^T \scal{\w_{1:}^t}{\E[\tell^t | \salgebra_{t-1}]-\tell^t}+\sum_{t=1}^T\scal{\w_{1:}^t}{\ell^t-\E[\tell^t | \salgebra_{t-1}]}\,.\]
    Using Lemma~\ref{lemma:azuma}, as $\scal{\w_{1:}^t}{\tell^t}\leq H$, we can bound the first term with probability $1-\delta/3$ by
    
    \[\sum_{t=1}^T \scal{\w_{1:}^t}{\E[\tell^t | \salgebra{t-1}]-\tell^t}\leq H\sqrt{2T\log(3/\delta)}\leq H\sqrt{2T\iota}\,.\]
    
    The second term can be upper-bounded by using the Lemma~\ref{lemma:balanced} and $\gamma_h=\gamma$ for the last equality,
    
    \begin{align*}
        \sum_{t=1}^T\scal{\w_{1:}^t}{l^t-\E[\tell^t | \salgebra_{t-1}]}&\leq\sum_{t=1}^T\sum_{h=1}^H\sum_{(x_h,a_h)\in\cAXh} p^{\nu^t}_{1:h}(x_h)\mu_{1:h}^t(x_h,a_h)\pa{1-\frac{\mu_{1:h}^t(x_h,a_h)}{\mu_{1:h}^t(x_h,a_h)+\gamma^\star_h(x_h,a_h)}}\\
        &=\sum_{t=1}^T\sum_{h=1}^H\sum_{(x_h,a_h)\in\cAXh}p^{\nu^t}_{1:h}(x_h)\frac{\mu_{1:h}^t(x_h,a_h)\gamma_h(x_h,a_h)}{\mu_{1:h}^t(x_h,a_h)+\gamma^\star_h(x_h,a_h)}\\
        &\leq \sum_{t=1}^T\sum_{h=1}^H\sum_{(x_h,a_h)\in\cAXh}p^{\nu^t}_{1:h}(x_h) \gamma^\star_h(x_h,a_h)\\
        &=\sum_{t=1}^T\sum_{h=1}^H\gamma_h\sum_{x_h\in\cX_h}\As{x_h} \frac{p^{\nu^t}_{1:h}(x_h)}{p^\star_{1:h}(x_h)}\\
        &=\gamma T \AX\,.
    \end{align*}
    
    In the setting with a fixed size of the action set, using the depth-wise inequality of Lemma~\ref{lemma:balanced}, and Lemma~\ref{lemma:fixed_size} for the last inequality, we can get instead 
    \begin{align*}
        \sum_{t=1}^T\scal{\w_{1:}^t}{l^t-\E[\tell^t | \salgebra_{t-1}]}&\leq\sum_{t=1}^T\sum_{h=1}^H\gamma_h\sum_{x_h\in\cX_h}A \frac{p^{\nu^t}_{1:h}(x_h)}{p^\star_{1:h}(x_h)}\\
        &\leq \gamma \AX\sum_{t=1}^T\sum_{h=1}^H A^{(H-h)/2}A^{h-H}\\
        &= \gamma \AX T\sum_{h=1}^H A^{(h-H)/2}\\
        &\leq (2+\sqrt{2})\gamma \AX T\,.
    \end{align*}
    
    In order to bound $\textrm{BIAS II}$, we can use Lemma~\ref{lemma:concentration}. Indeed for all $(x_h,a_h)\in\cAXh$, using $\delta'=\delta/(3\AX)$, the constant stopping time $\tau_h(x_h,a_h)=T$ and $\gamma'_h(x_h,a_h)=\gamma^\star_h(x_h,a_h)$, we have, as $\iota=\log(3A_{\cX}/\delta)$, with probability  at least $1-\delta/(3\AX)$,
    
    \[\tL_h^T(x_h,a_h)-L_h^T(x_h,a_h) \leq \frac{\iota}{2\gamma^\star_h(x_h,a_h)}\,.\]
    
    By a union bound the inequality holds for all $(x_h,a_h)\in\cAXh$ with probability at least  $1-\delta/3$. In this case, we get the bound, for any $\mu^\dagger\in\maxpi$, whith $\gamma_h=\gamma$
    \begin{align*}
    \sum_{t=1}^T\scal{\w_{1:}^\dagger}{\tell^t-\ell^t}
    &=\sum_{h=1}^H\sum_{(x_h,a_h)\in\cAXh} \mu_h^\dagger(x_h,a_h) \pa{\tL_h^T(x_h,a_h)-L_h^T(x_h,a_h)}\\
    &\leq \iota \sum_{h=1}^H\sum_{(x_h,a_h)\in\cAXh} \frac{\mu_h^\dagger(x_h,a_h)}{2\gamma^\star_h(x_h,a_h)}\\
    &=\frac{\iota}{2} \sum_{h=1}^H\frac{1}{\gamma_h}\sum_{(x_h,a_h)\in\cAXh} \mu_h^\dagger(x_h,a_h) p^\star_{1:h}(x_h)\\
    &=\frac{\iota}{2} \sum_{h=1}^H\frac{1}{\gamma_h}\\
    &=\frac{\iota}{2\gamma}H
    \end{align*}

which gives the bound on $\textrm{BIAS II}$ by taking the maximum over $\mu^\dagger\in\maxpi$. In the second setting we have instead, again with Lemma~\ref{lemma:fixed_size}

\[\sum_{t=1}^T\scal{\w_{1:}^\dagger}{\tell^t-\ell^t}\leq\frac{\iota}{2} \sum_{h=1}^H\frac{1}{\gamma_h}= \frac{\iota}{2\gamma}\sum_{h=1}^H A^{(h-H)/2}\leq (1+1/\sqrt{2})\frac{\iota}{\gamma}\,.\]
\end{proof}

Depending on the regularizer used in \BalancedFTRL, we then state some upper bounds of the \textrm{REG} term. 

\begin{lemma}{\label{lemma:bal_reg}} Assume that $\eta_h=\eta$, then with Tsallis entropy
\[\textrm{REG}\leq \frac{H^\tsallis}{\eta}\AX^{1-\tsallis}\,,\]
and with Shannon entropy
\[\textrm{REG}\leq \frac{H}{\eta}\log(\AX)\,.\]

Furthermore, if the size of the action set is fixed and $\eta_h=A^{(H-h)/2}\eta$ instead, we have with Shannon entropy
\[\textrm{REG}\leq \frac{2+\sqrt{2}}{\eta}\log(\AX)\,.\]

\end{lemma}

\begin{proof}
    The upper-bound with Tsallis entropy is a consequence of Hölder inequality, as for all $\w_{1:}\in\maxpibis$,
    
    \begin{align*}
        -\Psi(\w_{1:})&=\sum_{h=1}^H\sum_{(x_h,a_h)\in\cAXh}\frac{1}{\eta_h}\pa{p^\star_{1:h}(x_h) \mu_{1:h}(x_h,a_h)}^\tsallis\\
        &= \frac{1}{\eta}\sum_{h=1}^H\sum_{(x_h,a_h)\in\cAXh}\pa{p^\star_{1:h}(x_h) \mu_{1:h}(x_h,a_h)}^\tsallis\\
        &\leq \frac{1}{\eta}\pa{\sum_{h=1}^H\sum_{(x_h,a_h)\in\cAXh}p^\star_{1:h}(x_h) \mu_{1:h}(x_h,a_h)}^\tsallis \pa{\sum_{h=1}^H\sum_{(x_h,a_h)\in\cAXh}1}^{1-\tsallis}\\
        &= \frac{H^\tsallis}{\eta}\AX^{1-\tsallis}\,.
    \end{align*}
    
    When using Shannon entropy, we use that $p^\star_{1:h}\cdot\mu_{1:h}$ is a probability distribution on $\cAXh$ for all $h\in[H]$, and get with Jensen inequality
    \begin{align*}
        -\Psi(\w)&\leq\sum_{h=1}^H \frac{-1}{\eta_h}\Psi_h(p^\star_{1:h}.\mu_{1:h})\\
        &\leq \sum_{h=1}^H \frac{1}{\eta_h}\log(\AXh)\\
        &\leq \sum_{h=1}^H \frac{1}{\eta_h}\log(\AX)\\
        &=\frac{H}{\eta}\log(\AX)\,.
    \end{align*}
    
    If the size of the action set is fixed and $\eta_h=A^{(H-h)/2}\eta$, we instead get with Lemma~\ref{lemma:fixed_size}
    
    \[-\Psi(\w)\leq \sum_{h=1}^H \frac{1}{\eta_h}\log(\AX)= \sum_{h=1}^H \frac{A^{(h-H)/2}}{\eta}\log(\AX)\leq \frac{2+\sqrt{2}}{\eta}\log(\AX)\,.\]
\end{proof}

Upper bounding the \textrm{VAR} term is a little harder, but the idea is the same between the two regularizers, with each term of the sum over time being bounded separately in expectation.

\begin{lemma}\label{lemma:bal_var}

Let $v_t=\mathcal{D}_{\Psi^\star}\pa{\nabla\Psi(\w_{1:}^t)-\tell^t,\nabla\Psi(\w_{1:}^t)}$ for all $t\in[T]$. Then with a probability at least $1-\delta/3$,

\[\textrm{VAR}\leq  \sum_{t=1}^T\E[v_t | F_{t-1}]+H\sqrt{2T\iota}\,.\]

Furthermore, if $\eta_h=\eta$, we have with Tsallis entropy

\[\sum_{t=1}^T\E[v_t | F_{t-1}]\leq  \eta\frac{T\AX^\tsallis}{2\tsallis(1-\tsallis)}H^{1-\tsallis}\]

and with Shannon entropy

\[\sum_{t=1}^T\E[v_t | F_{t-1}]\leq  \frac{\eta}{2}T\AX\,.\]

If $\eta_h=A^{(H-h)/2}\eta$ and the size of the action set is fixed to $A$ instead, we then have with Shannon entropy
\[\sum_{t=1}^T\E[v_t | F_{t-1}]\leq  (1+1/\sqrt{2})\eta T\AX\,.\]

\end{lemma}

\begin{proof}

We can first notice, as $\Psi$ is supported on $\R_{\geq 0}^{\AX}$ and the estimated losses $\tell^t$ are non-negative, that

\[v_t=\scal{\w_{1:}^t}{\tell^t}+\Psi^\star\pa{\nabla\Psi(\w_{1:}^t)-\tell^t}-\Psi^\star\pa{\nabla\Psi(\w_{1:}^t)}\leq \scal{\w_{1:}^t}{\tell^t}\leq H\,.\]

Which lets us use Lemma~\ref{lemma:azuma} to get, with a probability at least $1-\delta/3$,

\[\sum_{t=1}^T \big(v_t-\E[v_t | \salgebra_{t-1}]\big)\leq \sqrt{2T\log(3/\delta)}\leq \sqrt{2T\iota}\,.\]

And obtain the first inequality with the Doob decomposition of the random process $(\sum_{k=1}^t v_k)_{t\in[T]}$, as $\textrm{VAR}=\sum_{k=1}^T v_k$ .

For the upper-bounds of $\sum_{t=1}^T\E\bra{v_t | F_{t-1}}$, we define for all $t\in[T]$, $f_t(u)=\mathcal{D}_{\Psi^\star}\pa{\nabla\Psi(\w_{1:}^t)-u\tell^t,\nabla\Psi(\w_{1:}^t)}$ for $u\in[0,1]$, such that $f_t(0)=0$ and $f_t(1)=v_t$. As for both entropy, $\Psi(\mu_{1:})$ (respectively $\Psi^\star(y)$) can be decomposed into $\Psi(\mu_{1:})=\sum_{h=1}^H\sum_{(x_h,a_h)\in\cAXh}\Psi_{x_h,a_h}(\mu_{1:h}(x_h,a_h))$ (respectively $\Psi^\star(y)=\sum_{h=1}^H\sum_{(x_h,a_h)\in\cAXh}\Psi^\star_{x_h,a_h}(y(x_h,a_h))$), the derivative of $f_t$ can be expressed with
\begin{equation}\label{eq:derivh}
    f_t'(u)=\sum_{h=1}^H\sum_{(x_h,a_h)\in\cAXh}\tell_h^t(x_h,a_h)\bra{\mu_{1:h}^t(x_h,a_h)-\nabla{\Psi^\star_{x_h,a_h}}\pa{\nabla\Psi_{x_h,a_h}(\mu_{1:h}^t(x_h,a_h))-u\tell^t_h(x_h,a_h)}}\,.
\end{equation}

When using Tsallis entropy, we have
\begin{align*}
    \nabla{\Psi_{x_h,a_h}}(\mu_{1:h}(x_h,a_h))&=- \frac{\tsallis}{\eta_h}p^\star_{1:h}(x_h)^{\tsallis}\mu_{1:h}(x_h,a_h)^{\tsallis-1}\quad\\
    \nabla{{\Psi^\star_{x_h,a_h}}}(y(x_h,a_h))&=\pa{-\frac{\eta_h y(x_h,a_h)}{\tsallis p^\star_{1:h}(x_h)^{\tsallis}}}^{\frac{-1}{1-\tsallis}}
\end{align*}

which yields
\begin{align*}
    \nabla{\Psi^\star_{x_h,a_h}}&\pa{\nabla\Psi_{x_h,a_h}(\mu_{1:h}^t(x_h,a_h))-u\tell^t_h(x_h,a_h)}\\
    &= \bra{\frac{-\eta_h}{\tsallis p^\star_{1:h}(x_h)^{\tsallis}}\pa{\frac{-\tsallis}{\eta_h}p^\star_{1:h}(x_h)\mu_{1:h}^t(x_h,a_h)^{\tsallis-1}-u\tell_h^t(x_h,a_h)}}^{\frac{-1}{1-\tsallis}}\\
    &=\mu_{1:h}^t(x_h,a_h)\bra{1-u\frac{\eta_h\tell_h^t(x_h,a_h)}{\tsallis p^\star_{1:h}(x_h)^{\tsallis}}\mu_{1:h}^t(x_h,a_h)^{1-\tsallis}}^{\frac{-1}{1-\tsallis}}\\
    &\geq \mu_{1:h}^t(x_h,a_h)\bra{1-u\frac{\eta_h\tell_h^t(x_h,a_h)}{\tsallis(1-\tsallis) p^\star_{1:h}(x_h)^{\tsallis}}\mu_{1:h}^t(x_h,a_h)^{1-\tsallis}}
\end{align*}

where we used $\pa{1+y}^{\frac{-1}{1-\tsallis}}\geq 1-\frac{y}{1-\tsallis}$ for $y\geq 0$. Plugging this in \eqref{eq:derivh}, along with $\tell_h^t(x_h,a_h)\mu_{1:h}^t(x_h,a_h)\leq \indic{x_h^t=x_h,a_h^t=a_h}$, gives

\begin{align*}
    f'_t(u)&\leq u\sum_{h=1}^H\sum_{(x_h,a_h)\in\cAXh}\tell_h^t(x_h,a_h)\mu_{1:h}^t(x_h,a_h)\frac{\eta_h\tell_h^t(x_h,a_h)}{\tsallis(1-\tsallis) p^\star_{1:h}(x_h)^{\tsallis}}\mu_{1:h}^t(x_h,a_h)^{1-\tsallis}\\
    &\leq u\sum_{h=1}^H\sum_{(x_h,a_h)\in\cAXh} \indic{x_h^t=x_h,a_h^t=a_h}\frac{\eta_h}{\tsallis(1-\tsallis) p^\star_{1:h}(x_h)^{\tsallis}}\mu_{1:h}^t(x_h,a_h)^{-\tsallis}\,.
\end{align*}

Summing as $u$ goes from $0$ to $1$ and conditioning yields
\begin{align*}
    \E[v_t | \salgebra_{t-1}]&\leq \sum_{h=1}^H\sum_{(x_h,a_h)\in\cAXh}\E\bra{ \indic{x_h^t=x_h,a_h^t=a_h}\frac{\eta_h}{2\tsallis(1-\tsallis) p^\star_{1:h}(x_h)^\tsallis}\mu_{1:h}^t(x_h,a_h)^{-\tsallis}|\salgebra_{t-1}}\\
    &=\sum_{h=1}^H\sum_{(x_h,a_h)\in\cAXh}\frac{\eta_h p_{1:h}^{\nu^t}(x_h)}{2\tsallis(1-\tsallis) p^\star_{1:h}(x_h)^{\tsallis}}\mu_{1:h}^t(x_h,a_h)^{1-\tsallis}\\
    &\leq \frac{\eta}{2\tsallis(1-\tsallis)}\pa{\sum_{h=1}^H\sum_{(x_h,a_h)\in\cAXh}\frac{p_{1:h}^{\nu^t}(x_h)}{p_{1:h}^\star(x_h)}}^\tsallis\pa{\sum_{h=1}^H\sum_{(x_h,a_h)\in\cAXh}p_{1:h}^{\nu^t}(x_h)\mu_{1:h}^t(x_h,a_h)}^{1-\tsallis}\\
    &=\frac{\eta H^{1-\tsallis}}{2\tsallis(1-\tsallis)}\AX^\tsallis
\end{align*}

where we used $\eta_h=\eta$, the Hölder inequality, and Lemma~\ref{lemma:balanced} for the last equality.

When using Shannon entropy, we have
\begin{align*}
    \nabla{\Psi_{x_h,a_h}}(\mu_{1:h}(x_h,a_h))&= \frac{p^\star_{1:h}(x_h)}{\eta_h}\bra{\log(p^\star_{1:h}(x_h)\mu_{1:h}(x_h,a_h))-1}\quad\\
    \nabla{{\Psi^\star_{x_h,a_h}}}(y(x_h,a_h))&=\exp\bra{\frac{\eta_h}{p_{1:h}^\star(x_h)}\pa{y(x_h,a_h)}+1-\log(p_{1:h}^\star(x_h))}
\end{align*}
and
\begin{align*}
    \nabla{\Psi^\star_{x_h,a_h}}&\pa{\nabla\Psi_{x_h,a_h}(\mu_{1:h}^t(x_h,a_h))-u\tell^t_h(x_h,a_h)}\\
    &= \exp\bra{\frac{\eta_h}{p_{1:h}^\star(x_h)}\pa{\frac{p^\star_{1:h}(x_h)}{\eta_h}\log(p^\star_{1:h}(x_h)\mu_{1:h}^t(x_h,a_h))-u\tell^t_h(x_h,a_h)}-\log(p_{1:h}^\star(x_h))}\\
    &=\mu_{1:h}^t(x_h,a_h)\exp{\bra{-u\frac{\eta_h\tell_h^t(x_h,a_h)}{p^\star_{1:h}(x_h)}}}\\
    &\geq \mu_{1:h}^t(x_h,a_h)\bra{1-u\frac{\eta_h\tell_h^t(x_h,a_h)}{p^\star_{1:h}(x_h)}}\,.
\end{align*}

Using this time  $1-y\geq \exp(-y)$ for all $y\geq 0$. With \eqref{eq:derivh}, this gives
\begin{align*}
    h_t'(u)&\leq u\sum_{h=1}^H\sum_{(x_h,a_h)\in\cAXh}\tell^t_h(x_h,a_h)\mu_{1:h}^t(x_h,a_h)\frac{\eta_h\tell_h^t(x_h,a_h)}{p^\star_{1:h}(x_h)}\\
    &\leq u\sum_{h=1}^H\sum_{(x_h,a_h)\in\cAXh}\indic{x_h^t=x_h,a_h^t=a_h}\frac{\eta_h}{p^\star_{1:h}(x_h)\mu_{1:h}^t(x_h,a_h)}\,.
\end{align*}

Again, summing from $0$ to $1$, using $\eta_h=\eta$ and Lemma~\ref{lemma:balanced}, leads to
\begin{align*}
    \E[v_t | \salgebra_{t-1}]&\leq \sum_{h=1}^H\sum_{(x_h,a_h)\in\cAXh}\frac{\eta_h p^{\nu^t}_{1:h}(x_h,a_h)}{2p^\star_{1:h}(x_h)}\\
    &=\frac{\eta}{2} \AX\,.
\end{align*}

If the size of the action set is fixed and $\eta_h=A^{(H-h)/2}\eta$, we have instead, still with Shannon entropy,
\begin{align*}
    \E[v_t | \salgebra_{t-1}]&\leq \sum_{h=1}^H\sum_{(x_h,a_h)\in\cAXh}\frac{\eta_h p^{\nu^t}_{1:h}(x_h,a_h)}{2p^\star_{1:h}(x_h)}\\
    &\leq \frac{\eta}{2}\AX \sum_{h=1}^H A^{(h-H)/2}\\
    &\leq (1+1/\sqrt{2})\eta\AX
\end{align*}

where we used the depth-wise inequality of Lemma~\ref{lemma:balanced} and Lemma~\ref{lemma:fixed_size}.

\end{proof}

We can now prove the main theorem of Section~\ref{sec:tree_structure_known} using the previous lemmas.

\thmbal*

\begin{proof}

    From the decomposition of Lemma~\ref{lemma:bal_decompos}, we have
    
    \[\regret_\mathrm{max}^T\leq \textrm{BIAS I}+\textrm{BIAS II}+\textrm{REG}+\textrm{VAR}\,.\]
    
    Lemma~\ref{lemma:bal_bias} directly yields with probability  at least $1-2\delta/3$
    
    \[\textrm{BIAS I}+\textrm{BIAS II}\leq H\sqrt{2T\iota} + \gamma T \AX+H\frac{\iota}{2\gamma}\]
    
    that is minimized with $\gamma=\sqrt{{H\iota}/{(2\AX T)}}$.
    
    Then, using Tsallis entropy we have for the two other terms, from Lemma~\ref{lemma:bal_bias} and Lemma~\ref{lemma:bal_var}, with probability at least $1-\delta/3$
    \[\textrm{REG}+\textrm{VAR}\leq H\sqrt{2\iota T}+\frac{H^\tsallis}{\eta}\AX^{1-\tsallis}+\eta\frac{T\AX^\tsallis}{2\tsallis(1-\tsallis)}H^{1-\tsallis}\]
    
    minimized with $\eta=\sqrt{{2\tsallis(1-\tsallis)}/{T}}\:H^{\tsallis-1/2}\AX^{1/2-\tsallis}$. With the trivial upperbound $H\leq \AX$, we then get by summing the two previous inequalities, with a probability at least $1-\delta$
    
    \[\regret_\mathrm{max}^T\leq 3\sqrt{2H\iota \AX T}+\sqrt{\frac{2}{\tsallis(1-\tsallis)}H\AX T}\,.\]
    
    When using Shannon entropy, we instead have (with a probability at least $1-\delta/3$) from the same lemmas:
    
    \[\textrm{REG}+\textrm{VAR}\leq H\sqrt{2\iota T}+\frac{H}{\eta}\log(\AX)+\frac{\eta}{2} \AX T\]
    
    minimized with $\eta=\sqrt{{2H\log(\AX)}\pa{T\AX}}$, which in this case yields by summing, again with probability at least $1-\delta$,
    
    \[\regret_\mathrm{max}^T\leq 3\sqrt{2HT\AX\iota}+\sqrt{2HT\AX\log(\AX)}\,.\]
\end{proof}

In the special case of a fixed action set considered in the previous lemma, we can get a better rate for the regret of the max-player, here shown for Shannon entropy.

\thmfixed*

\begin{proof}
    The idea is the same as in the previous theorem, but using instead the upper bounds obtained with a fixed size of the action set that all holds with probability at least $1-\delta$, as both $\eta_h$ and $\gamma_h$ parameters are in the form $\eta_h=A^{(H-h)/2}\eta$ and $\gamma_h=A^{(H-h)/2}\gamma$. The case $A=1$ is ignored, as the regret is then trivially $0$. With Lemma~\ref{lemma:bal_bias}, we now have:
    
    \[\textrm{BIAS I}+\textrm{BIAS II}\leq H\sqrt{2T\iota} + (2+\sqrt{2})\gamma T \AX+(1+1/\sqrt{2})\frac{\iota}{\gamma}\]
    
    minimized with $\gamma=\sqrt{\iota/(2T\AX)}$. The $\textrm{REG}$ and $\textrm{VAR}$ terms are in this case upper bounded, using Lemma~\ref{lemma:bal_reg} and Lemma~\ref{lemma:bal_var}, by
    
    \[\textrm{REG}+\textrm{VAR}\leq H\sqrt{2T\iota}+\frac{2+\sqrt{2}}{\eta}\log(\AX)+(1+1/\sqrt{2})\eta T\AX\]
    
    minimized with $\eta=\sqrt{2\log(\AX)/(T\AX)}$. By summing everything, we now get, using this time  the better inequality $H\leq \sqrt{\AX}$ from Lemma~\ref{lemma:fixed_size},
    
    \[\regret_{\mathrm{max}}^T\leq (2+4\sqrt{2})\sqrt{T\AX\iota} +(2+2\sqrt{2})\sqrt{T\AX\log(\AX)}\]
    
    and we conclude using $2\sqrt{2}\leq 3$ for readability.
\end{proof}

\section{Efficient updates and proof of correctness}
\label{app:algorithmic_update}

\paragraph{Dilated entropy equivalence} We first state a useful property related to the dilated divergence. For any list of learning rates $\eta^r:=(\eta^r_h(x_h))_{h,x_h}$, we define the dilated entropy $\Psi_{\eta^r}$ by

\[\Psi_{\eta^r}(\mu_{1:}):=\sum_{h=1}^H\sum_{(x_h,a_h)\in \cAXh}\frac{\mu_{1:h}(x_h,a_h)}{\eta^r_h(x_h)}\log\pa{\frac{\mu_{1:h}(x_h,a_h)}{\sum_{a'_h\in\cA(x_h)}\mu_{1:h}(x_h,a'_h)}}\,.\]

The following lemma is the straight-forward extension of the Lemma 9 by \citet{kozuno2021learning} to non-constant learning rates.

\begin{lemma}\label{lemma:entropy_div}
   For any vector $\eta^r$ of learning rates and $(\mu^1,\mu^2)\in\maxpi^2$, the Bregman divergence $\mathcal{D}_{\Psi_{\eta^r}}(\mu^1_{1:},\mu^2_{1:})$ between the realization plans coincides with the dilated divergence $\mathcal{D}_{\eta^r}(\mu^1,\mu^2)$ defined in Section \ref{prg:dilated}.
\end{lemma}

The proof is essentially the same, and is based on the fact (later used in Proposition \ref{prop:update}) that for all positive $\mu\in\maxpi$
\begin{equation}
    \nabla_{x_h,a_h}\Psi_{\eta^r}(\mu_{1:})=\log(\mu_h(a_h|x_h)) \label{eq:deriv_entropy}\,.
\end{equation}
We then state the result on the equivalence between update~\eqref{eqn:U1} and update~\eqref{eqn:U2}.

\begin{proposition}\label{prop:equiv_u12}
    Let $(\mu^t)_{t\in [|0,T|]}$ obtained using \BalancedFTRL with Shannon entropy and any sequence $(\eta_h)_{h\in [H]}$ of learning rates. Then there exists a vector $\eta^\star$ of learning rates and policy $\mu^\star\in\maxpi$ that can be computed with a time complexity $\cO(\AX)$ such that it coincides with update~\eqref{eqn:U2},
    \[\mu^t =\argmin_{\mu\in \maxpi} \scal{\mu}{\tilde{L}^{t-1}}+ D_{\eta^\star}\pa{\mu,\mu^\star}\,.\]
\end{proposition}

\begin{proof}
    From update \eqref{eqn:U1}, $\mu^t=\argmin_{\mu\in\maxpi} \scal{\mu_{1:}}{\tL^{t-1}}+\Psi^1(\mu_{1:})$ where
    
     \[\Psi^1(\mu_{1:})=\sum_{h=1}^H\sum_{(x_h,a_h)\in\cAXh}\frac{p^\star_{1:h}(x_h)\mu_{1:h}(x_h,a_h)}{\eta_h}\log\pa{p^\star_{1:h}(x_h) \mu_{1:h}(x_h,a_h)}\]
     we first notice that
     
     \begin{align*}
         \Psi^1&(\mu_{1:})-\sum_{h=1}^H\sum_{(x_h,a_h)\in\cAXh}\frac{p^\star_{1:h}(x_h)\mu_{1:h}(x_h,a_h)}{\eta_h}\log\pa{p^\star_{1:h}(x_h)}\\
         &=\sum_{h=1}^H\sum_{(x_h,a_h)\in\cAXh}\frac{p^\star_{1:h}(x_h)\mu_{1:h}(x_h,a_h)}{\eta_h}\log(\mu_{1:h}(x_h,a_h))\\
         &=\sum_{h=1}^H\sum_{(x_h,a_h)\in\cAXh}\sum_{h'=1}^h\frac{p^\star_{1:h}(x_h)\mu_{1:h}(x_h,a_h)}{\eta_h}\log(\mu_{h'}(a_{h'}|x_{h'}))\\         &=\sum_{h'=1}^H\sum_{(x_{h'},a_{h'})\in\cAs{\cX_{h'}}}\log(\mu_{h'}(a_{h'}|x_{h'}))\sum_{h=h'}^H\sum_{(...,x_{h'},a_{h'},...,x_h,a_h)}\frac{p^\star_{1:h}(x_h)\mu_{1:h}(x_h,a_h)}{\eta_h}\\
         &=\sum_{h'=1}^H\sum_{(x_{h'},a_{h'})\in\cAs{\cX_{h'}}}p^\star_{1:{h'}}(x_{h'})\mu_{1:{h'}}(x_{h'},a_{h'})\log(\mu_{h'}(a_{h'}|x_{h'}))\sum_{h=h'}^H\frac{1}{\eta_{h}}\\
         &=\sum_{h'=1}^H\sum_{(x_{h'},a_{h'})\in\cAs{\cX_{h'}}}\frac{\mu_{1:h'}(x_{h'},a_{h'})}{\eta^\star_{h'}}\log(\mu_{h'}(a_{h'}|x_{h'}))\\
         &=\Psi_{\eta^\star}(\mu_{1:})
     \end{align*}
     
where we defined 

\[\eta^\star_{h'}(x_{h'}):=1/\pa{p^\star_{1:h'}(x_{h'})\sum_{h=h'}^H\frac{1}{\eta_h}}\,.\]

This shows that $\Psi^1$ only differs from $\Psi_{\eta^\star}$ by a linear term of the realization plan: the Bregman divergence $\mathcal{D}_{\Psi^1}$ of $\Psi^1$ is thus exactly $\mathcal{D}_{\Psi_{\mu^\star}}$. We can then conclude using Lemma~\ref{lemma:entropy_div} with $\mu^\star_{1:}:=\argmin_{\mu^\star\in\maxpibis} \Psi'(\mu^\star_{1:})$

\[\Psi^1(\mu_{1:})=\mathcal{D}_{\Psi^1}(\mu_{1:})+\Psi^1(\mu^\star_{1:})=\mathcal{D}_{\Psi_{\mu^\star}}(\mu,\mu^\star)+\Psi'(\mu^\star_{1:})=\mathcal{D}_{\mu^\star}(\mu,\mu^\star)+\Psi'(\mu^\star_{1:})\,.\]

This shows that the two updates minimize two functions of $\mu$ that only differ by a constant terms, and are thus equivalent.

The calculation of the $\eta^\star$ is straight-forward with the above formula, while the policy $\mu^\star$ is the minimum over a dilated entropy function that can be recursively computed in $\cO(\AX)$, starting from the leaves. 
\end{proof}

We now state Algorithm \ref{al:adapt_update} and a proof of its correctness.

\begin{algorithm}[t]
\caption{Fast Tree Update}
\label{al:adapt_update}
\begin{algorithmic}[1]
            \STATE \textbf{Input:}\\
			~~~~ $\mu^{t-1}$ given by update \eqref{eqn:U2}\\
			~~~~ Trajectory $(x_1^{t-1},a_1^{t-1}, ..., x_H^{t-1},a_H^{t-1})$\\
			~~~~ Loss $\tell^{t-1}$, old and new learning rates $\eta^{t-1}$ and $\eta^t$ along the trajectory\\
			~~~~ Base policy $\mu^0$\\
			\STATE \textbf{Init:}\\
			~~~~ $Z_{H+1}^{t-1}\gets 1$\\
			~~~~ $\alpha_h^{t-1}=\eta_h^t(x_h^{t-1})/\eta_h^{t-1}(x_h^{t-1})$
			\STATE \textbf{Iterate:}\\
			\textbf{For} $h=H$ to $1$\\\vspace{.05cm}
			~~~~ $\lell_h^{t-1}\gets \tell_h^{t-1}(x_h^{t-1},a_h^{t-1})-\log(Z_{h+1}^{t-1})/\eta_{h+1}^t(x_{h+1}^{t-1})$\\
			~~~~ $Z_h^{t-1}\gets \sum_{a_h\in\cA\pa{x_h^t}}\mu_h^t(a_h|x_h^{t-1})^{\alpha_h^{t-1}}\mu_h^0(a_h|x_h^{t-1})^{1-\alpha_h^{t-1}}\exp\pa{-\indic{a_h=a_h^{t-1}}\eta_h^t(x_h^{t-1})\lell_h^t}$\\
			~~~~ \textbf{For} all $a_h\in \cA\pa{x_h^{t-1}}$ :\\
			~~~~~~~~ $\mu_{h}^{t}(a_h|x_h^{t-1}) \gets \mu_h^{t-1}(a_h|x_h^{t-1})^{\alpha_h^{t-1}}\mu_h^0(a_h|x_h^{t-1})^{1-\alpha_h^{t-1}}\exp\pa{-\log(Z_h^{t-1})-\indic{a_h=a_h^{t-1}}\eta_h^t(x_h^{t-1})\lell_h^t}$
			
			\textbf{For} all other information sets $x_h$ :\\
			~~~~ $\mu_{h}^{t}(. | x_h)\gets \mu_{h}^{t-1}(. | x_h)$\\\vspace{.05cm}
                \STATE \textbf{Output:} $\mu^t$ 
			
\end{algorithmic}
\end{algorithm}	

\begin{proposition}\label{prop:update}
    Assume that for all $t\in[T]$, the vector $\eta^t$ of learning rates is only updated along the trajectory of episode $t$. Let $(\mu^t)_{t\in[T]}$ be the sequence of policies computed with update \eqref{eqn:U2}. Then it coincides with the sequence of policies that would be computed with Algorithm \ref{al:adapt_update}.
\end{proposition}

\begin{proof}
    We remind that update \eqref{eqn:U2} define the sequence $(\mu^t)_{t\in[T]}$ with
    \[\mu^t=\textrm{argmin}_{\mu\in \maxpi} \scal{\mu_{1:}}{\tL^{t-1}} +\mathcal{D}_{\eta^t}(\mu,\mu^0)\,.\]
    Using the decomposition $\maxpibis= (F+u)\cap\R_{\geq 0}^{\AX}$, Lemma~\ref{lemma:entropy_div} (on the equivalence between $\mathcal{D}_{\eta^t}$ and $\mathcal{D}_{\Psi_{\eta^t}}$), we then know the existence for all $t\in[T]$ of $g^t\in F^\perp$ such that
    \[\nabla\Psi_{\eta^t}(\mu^t)-\nabla\Psi_{\eta^t}(\mu^0)-\tL^{t-1}=g^t\,.\]
    Taking the difference on the last equation between $t$ and $t-1$ yields:
    \begin{equation}\label{eq:nabla_alg}
        \nabla\Psi_{\eta^t}(\mu^t)=\nabla\Psi_{\eta^{t-1}}(\mu^{t-1})+\Psi_{\eta^{t-1}}(\mu^0)-\Psi_{\eta^t}(\mu^0)-\tl^{t-1}+\Delta_t
    \end{equation}
    where $\Delta_t:=g_t-g_{t-1}\in F^\perp$. And from the normalizing constraints on $\maxpibis$, $\Delta_t$ is characterized by the vector $\lambda^t=(\lambda^t(x))_{x\in\cX}$ with
    \[\Delta_t(x_h,a_h)=-\lambda^t(x_h)+\sum_{(x_{h+1})\in \cX_{h+1}(x_h,a_h)}\lambda(x_{h+1})\] with the sum on the second term being null when $h=H$.
    
    We first show by backward induction on $h\in[H]$ that for all $x_h$ not visited at
    episode $t-1$, $\lambda^t(x_h)=0$. Indeed, in this case $\eta_h^{t-1}(x_h)=\eta_h^t(x_h)$, and the induction property leads to $\Delta^t(x_h,a_h)=-\lambda^t(x_h)$ for all $a_h$, as the perfect recall assumption implies any $x_{h+1}$ is also not visited at episode $t$. For all $a_h\in\cA(x_h)$, using equation~\eqref{eq:nabla_alg} on the component $(x_h,a_h)$  gives with the expression of the gradient of equation~\eqref{eq:deriv_entropy}
    \[\frac{1}{\eta_h^t(x_h)}\log(\mu^t_h(a_h|x_h))=\frac{1}{\eta_h^{t}(x_h)}\log(\mu^{t-1}_h(a_h|x_h))-\lambda^t(x_h)\]
    which is equivalent to
    \[\mu^t_h(a_h|x_h))=\mu^{t-1}_h(a_h|x_h)\exp\pa{-\eta^t_h(x_h)\lambda_h^t(x_h)}\,.\]
    Summing over $a_h\in\cAs{x_h}$ yields $\eta^t_h(x_h)\lambda_h^t(x_h)=0$ and thus $\lambda_h^t(x_h)=0$ as $\eta^t_h(x_h)\neq 0$, which concludes the induction. We finally get from the previous equality $\mu_h^t(.|x_h)=\mu_h^{t-1}(.|x_h)$ for all $x_h$ outside of the trajectory at time $t-1$, as in Algorithm \ref{al:adapt_update}.
    
    For all $h\in[H]$, we are then interested in the visited state $x_h^{t-1}$, and define by convention $\lambda^t(x_{H+1}^t)=0$. Then we have for $a_h\in\cAs{x_h^{t-1}}/\{a_h^{t-1}\}$, $\Delta^t(x_h^{t-1},a_h)=-\lambda^t(x_h^{t-1})$ and $\Delta^t(x_h^{t-1},a_h^{t-1})=-\lambda^t(x_h^{t-1})+\lambda^t(x_{h+1}^{t-1})$ using the previous property that $\lambda^t$ is null on unvisited states. For all $a_h\in\cAs{x_h^{t-1}}$, we now have with equations \eqref{eq:deriv_entropy} and \eqref{eq:nabla_alg}:
    \begin{align*}
        \frac{1}{\eta_h^{t}(x_h^{t-1})}\log(\mu_h^t(a_h|x_h^{t-1}))=\frac{1}{\eta_h^{t-1}(x_h^{t-1})}&\log(\mu_h^{t-1}(a_h|x_h^{t-1}))+\pa{\frac{1}{\eta_h^{t}(x_h^{t-1})}-\frac{1}{\eta_h^{t-1}(x_h^{t-1})}}\log(\mu_h^0(a_h|x_h^{t-1}))\\
        &+\indic{a_h=a_h^{t-1}}\pa{\lambda^t(x_{h+1}^{t-1})-\tl_h^{t-1}(x_h^{t-1},a_h^{t-1})}-\lambda^t(x_h^{t-1})\,.
    \end{align*}
    Multiplying by $\eta_h^{t}(x_h^{t-1})$ and taking the exponential then leads to:
    \[\mu_{h}^{t}(a_h|x_h^{t-1})= \mu_h^{t-1}(a_h|x_h^{t-1})^{\alpha_h^{t-1}}\mu_h^0(a_h|x_h^{t-1})^{1-\alpha_h^{t-1}}\exp\pa{-\eta_h^{t}(x_h^{t-1})\lambda^t(x_h^{t-1})-\indic{a_h=a_h^{t-1}}\eta_h^t(x_h^{t-1})\lell_h^t}\]
    where $\lell_h^{t-1}=\tell_h^{t-1}(x_h^{t-1},a_h^{t-1})-\lambda^t(x_{h+1}^{t-1})$ and $\alpha_h^{t-1}=\eta_h^t(x_h^{t-1})/\eta_h^{t-1}(x_h^{t-1})$. We again recognize the update of algorithm \ref{al:adapt_update} with $Z_h^{t-1}=\exp\pa{\eta_h^t(x_h^{t-1})\lambda^t(x_h^{t-1})}$, as summing the previous equation over $a_h\in\cAs{x_h^{t-1}}$ yields:
    \[Z_h^{t-1}=\sum_{a_h\in\cA\pa{x_h^{t-1}}}\mu_h^t(a_h|x_h^{t-1})^{\alpha_h^{t-1}}\mu_h^0(a_h|x_h^{t-1})^{1-\alpha_h^{t-1}}\exp\pa{-\indic{a_h=a_h^{t-1}}\eta_h^t(x_h^{t-1})\lell_h^t}\,.\]

    We therefore obtained that policy $\mu^t$ obtained with update \eqref{eqn:U2} coincides with the output of Algorithm \ref{al:adapt_update} both in visited and unvisited information set at episode $t-1$, which concludes the proof. 
\end{proof}

\section{Proof of \texorpdfstring{\AdaptiveFTRL}{Adaptive-FTRL}}
\label{app:adaptative_FTRL_proof}

In all this section, we will use the constants $\iota'=\log(3\AX T/\delta)$ and $\logtwo=1+\log_2(T)$ of theorem \ref{thm:main_ada} for readability.

This first lemma centralizes the necessary inequalities for the analysis of the regret which can all hold with a high probability. 

\begin{lemma}\label{lemma:approx_ada}
    We assume $\iota'/\gamma\leq 1/2$ and $\gamma\geq 1$. Then, when using \AdaptiveFTRL, the following inequalities all hold with a probability at least $1-\delta$:
    \begin{align}
        \tP_{1:}^T(x_h,a_h)&\leq 1+2P_{1:}^T(x_h) \label{p_ineq}\,,\\      
        \tL^T(x_h,a_h)-L^T(x_h,a_h)&\leq 1+\frac{2\iota'}{\gamma}P_{1:}^T(x_h)\label{l_ineq}\,,\\
        \sum_{t=1}^T \scal{\mu_{1:}^t}{\ell^t-\tell^t}&\leq H\sqrt{2T\iota'} \label{azuma_ineq}\,.
    \end{align}
    We denote by $\eventadapt$ the event under which all these inequalities hold.
\end{lemma}

\begin{proof}

Let $(x_h,a_h)\in\cAXh$. We can first notice that for all $0\leq t\leq T$, $1+\tP_{1:h}^t(x_h,a_h)\leq 2^t$. Indeed, $1+\tP_a^0=1$ and recursively, if $1+\tP_{1:h}^t(x_h,a_h)\leq 2^t$, then $\tp_{1:h}^{t+1}(x_h,a_h)\leq (1+\tP_{1:h}^t(x_h,a_h))/\gamma\leq 2^t/\gamma\leq 2^t$, which implies $1+\tP_{1:h}^{t+1}(x_h,a_h)\leq 2^{t+1}$.

In particular, we get $\gamma_h^T(x_h,a_h)\geq \gamma/2^T$, and define for all $i\in[T]$, $\gammascale{i}=\gamma/2^i$.

Let $\tau^i_h(x_h,a_h)=\max\left\{t\leq T, \gamma_h^t(x_h,a_h)\geq\gammascale{i}\right\}$. As $\gamma_h^t(x_h,a_h)$ is a decreasing predictable sequence for $\mathcal{F}_t$, $\tau^i_h(x_h,a_h)$ is a stopping time, which lets us apply Lemma~\ref{lemma:concentration} with $\tau_h(x_h,a_h)=\tau^i_h(x_h,a_h)$, $\delta'=\delta/(3\AX T)$ and $\gamma'_h(x_h,a_h)=\gammascale{i}$ to get with a probability at least $1-\delta'$,

\[\tL_h^{\tau^i_h(x_h,a_h)}(x_h,a_h)-L_h^{\tau^i_h(x_h,a_h)}(x_h,a_h)\leq \frac{\iota'}{2\gammascale{i}}\]

where $\iota'=\log(3\AX T/\delta)$. We especially get that this inequality holds for all $(x_h,a_h)\in\cAXh$ and $i\in[T]$ with a probability at least $1-\delta/3$. As $\gamma/2^{T}\leq\gamma_h^T(x_h,a_h)\leq \gamma$, we know the existence of $i(x_h,a_h)$, such that $\gammascale{i(x_h,a_h)}\leq \gamma_h^T(x_h,a_h)\leq 2\gammascale{i(x_h,a_h)}$, and the previous inequality used with $i(x_h,a_h)$ then gives:

\[\tL_h^T(x_h,a_h)-L_h^T(x_h,a_h)\leq \frac{\iota'}{2\gammascale{i}}\leq \frac{\iota'}{\gamma_h^T(x_h,a_h)}\,.\]

As the definition of $i\pa{x_h,a_h}$ implies $\tau_h^{i\pa{x_h,a_h}}(x_h,a_h)=T$

The same exact steps (along with Lemma~\ref{lemma:concentration}) can also be performed with $\tp_{1:h}^t(x_h,a_h)$, and we similarly obtain that with probability at least $1-\delta/3$, for all $(x_h,a_h)\in\cAXh$,

\[\tP_{1:h}^T(x_h,a_h)-P_{1:h}^T(x_h,a_h)\leq \frac{\iota'}{\gamma_h^T(x_h,a_h)}\]

which gives, using the definition of $\gamma_h^T(x_h,a_h)$,

\[\tP_{1:h}^T(x_h,a_h)-P_{1:h}^T(x_h,a_h)\leq \frac{\iota'}{\gamma}\pa{1+\tP_{1:h}^{T-1}(x_h,a_h)}\leq \frac{\iota'}{\gamma}\pa{1+\tP_{1:h}^T(x_h,a_h)}\]

and, as $\frac{\iota'}{\gamma}\leq 1/2$ by assumption,

\[\tP_{1:h}^T(x_h,a_h)\leq \frac{2\iota'}{\gamma}+2P_{1:h}^T(x_h,a_h)\leq 1+2P_{1:h}^T(x_h,a_h)\]

and we recognize inequality~\eqref{p_ineq}. Plugging this back in the previous inequalities on $\tL_h^T$ also yields

\[\tL_h^T(x_h,a_h)-L_h^T(x_h,a_h)\leq\frac{\iota'}{\gamma}\pa{1+\tP_{1:h}^T(x_h,a_h)}\leq\frac{\iota'}{\gamma}\pa{2+2 P_{1:h}^T(x_h,a_h)}\leq 1+\frac{2\iota'}{\gamma}P_{1:h}^T(x_h,a_h)\]

and we now recognize inequality~\eqref{l_ineq}.

The last inequality is a consequence of the Lemma~\ref{lemma:azuma} used on $u_t=\scal{\mu_{1:}^t}{\hl^t}$. Indeed, as $u_t$ is non-negative and $\E\bra{\scal{\mu_{1:}^t}{\hl^t}| \salgebra_{t-1}}=\scal{\w_{1:}^t}{\ell^t}\leq H$, the lemma gives with a probability at least $1-\delta/3$

\[\sum_{t=1}^T \scal{\w_{1:}^t}{\ell^t-\tell^t}\leq H\sqrt{2T}\log(3/\delta)\leq H\sqrt{2T\iota'}\,.\]

As the three inequalities each hold with a probability at least $1-\delta/3$, they all hold at the same time with a probability at least $1-\delta$, which concludes the proof.
\end{proof}

As in Section \ref{appendix:bal}, we decompose the regret between the \textrm{BIAS}, \textrm{REG} and \textrm{VAR} terms. Note that the analysis of \textrm{REG} and \textrm{VAR} terms are different, although conceptually the same. The proof is directly based on the algorithmic update~\eqref{al:adapt_update}, in which the sequences $(Z_h^t)_{h\in[H]}$ are defined.

\[\]

\begin{lemma}\label{lemma:ada_decompos}
    Using \AdaptiveFTRL yields
    \begin{align*}
        \regret_\mathrm{max}^T &\leq \underbrace{\sum_{t=1}^T\scal{\mu_{1:}^t}{\ell^t-\tell^t}}_{\textrm{BIAS I}}+\underbrace{\max_{\mu_{1:}^\dagger\in\maxpi}\sum_{t=1}^T\scal{\w^\dagger}{\tell^t-\ell^t}}_{\textrm{BIAS II}}\\
        &\qquad+\underbrace{\max_{\mu^\dagger\in\maxpi}\mathcal{D}_{\eta^{T+1}}\pa{\mu^\dagger,\mu^0}}_{\textrm{REG}}+\underbrace{\sum_{t=1}^T\bra{\scal{\mu_{1:}^t}{\tell^t}+\frac{1}{\eta_1^{t+1}(x_1^t)}\log\pa{Z_1^t}}}_{\textrm{VAR}}\,.
    \end{align*} 
\end{lemma}

\begin{proof}
Let $t\in[T]$ and $h\in[H]$. From Algorithm \ref{al:adapt_update}, we have for all $a_h\in\cAs{x_h^t}$,
    \[\mu_h^{t+1}(x_h^t,a_h)=\mu_h^t(a_h|x_h^t)^{\alpha_h^t}\mu_h^0(a_h|x_h^t)^{1-\alpha_h^t}\exp\pa{-\log(Z_h^t)-\indic{a_h=a_h^t}\eta_h^{t+1}(x_h^t)\lell_h^t}\]
    
    where $\lell_h^t=\tell_h^t-\frac{\log{Z_{h+1}^t}}{\oeta_{h+1}^t}$. We then get, for all $\mu^\dagger\in\maxpi$,
    
    \begin{align*}
        \sum_{a_h\in\cAs{x_h^t}}\mu_{1:h}^\dagger(x_h^t,a_h)\log(\mu_{h}^{t+1}(a_h|x_h^t))=&\sum_{a_h\in\cAs{x_h^t}}\mu_{1:h}^\dagger(x_h,a_h)\bra{\alpha_h^t\log(\mu_{h}^t(a_h|x_h^t))+(1-\alpha_h^t)\log(\mu_h^0(a_h|x_h^t))-\log(Z_h^t)}\\
    &-\eta_h^{t+1}(x_h^t)\bra{\w_{1:h}^{\dagger}(x_h^t,a_h^t)\tell_h^t+\frac{\eta_{h}^{t+1}(x_h^t)}{\mu_{1:h}^{\dagger}(x_h^t,a_h)}\log(Z_{h+1}^t)}\,.
    \end{align*}
    
    Using this equality then gives, with the convention $\mu_{1:0}^\dagger(.)=1$,
    
    \begin{align*}
        -\scal{\w^\dagger_{1:}}{\tell^t}&=-\sum_{h=1}^H\w_{1:h}^\dagger(x_h^t,a_h^t)^{\dagger}\tell_h^t\\
        &=\sum_{h=1}^H\sum_{a_h\in\cAs{x_h^t}}\frac{\mu_{1:h}^\dagger(x_h,a_h)}{\eta_h^{t+1}(x_h^t)}\bra{\log(\mu_h^{t+1}(a_h|x_h^t))-\alpha_h^t\log(\mu_h^t(a_h|x_h^t))-(1-\alpha_h^t)\log(\mu_h^0(a_h|x_h^t))+\log(Z_h^t)}\\
        &\qquad-\sum_{h=1}^H\frac{\mu_{1:h}^\dagger(a_h^t|x_h^t)}{\eta_{h+1}^{t+1}}\log(Z_{h+1}^t)\\
        &=\sum_{h=1}^H\sum_{a_h\in\cAs{x_h^t}}\frac{\mu_{1:h}^\dagger(x_h,a_h)}{\eta_h^{t+1}(x_h^t)}\bra{-\log\pa{\frac{\mu_h^\dagger(a_h|x_h^t)}{\mu_{h}^{t+1}(a_h|x_h^t)}}+\alpha_h^t\log\pa{\frac{\mu_h^\dagger(a_h|x_h^t)}{\mu_{h}^t(a_h|x_h^t)}}+(1-\alpha_h^t)\log\pa{\frac{\mu_h^\dagger(a_h|x_h^t)}{\mu_{h}^0(a_h|x_h^t)}}}\\
        &\qquad+\sum_{h=1}^H\bra{\frac{\mu_{1:{h-1}}^\dagger(a_{h-1}^t|x_{h-1}^t)}{\eta_h^{t+1}(x_h^t)}\log(Z_h^t)-\frac{\mu_{1:h}^\dagger(a_h^t|x_h^t)}{\eta_{h+1}^{t+1}(x_{h+1}^t)}\log(Z_{h+1}^t)}\\
        &=\sum_{h=1}^H \sum_{a_h\in\cAs{x_h^t}}\mu_{1:h}^\dagger(x_h,a_h)\left[\frac{-1}{\eta_h^{t+1}(x_h^t)}\log\pa{\frac{\mu_h^\dagger(a_h|x_h^t)}{\mu_{h}^{t+1}(a_h|x_h^t)}}+\frac{1}{\eta_h^{t}}\log\pa{\frac{\mu_h^\dagger(a_h|x_h^t)}{\mu_{h}^t(a_h|x_h^t)}}\right.\\
        &\qquad+\left.\pa{\frac{1}{\eta_h^{t+1}(x_h^t)}-\frac{1}{\eta_h^{t}(x_h^t)}}\log\pa{\frac{\mu_h^\dagger(a_h|x_h^t)}{\mu_{h}^0(a_h|x_h^t)}}\right]
        +\frac{1}{\eta_{1}^{t+1}(x_1^t)}\log(Z_1^t)\,.
    \end{align*}
    
    Then, using that $\eta_h^t(x_h)$ only updates when $x_h=x_h^t$ and in the second equality telescoping, with $\mu^1=\mu^0$, we get
    
    \begin{align*}
        -\sum_{t=1}^T\scal{\w^\dagger_{1:}}{\tell^t}&=\sum_{t=1}^T\sum_{h=1}^H\sum_{(x_h,a_h)\in\cAXh}\mu_{1:h}^\dagger(x_h,a_h)\left[-\frac{-1}{\eta_h^{t+1}(x_h)}\log\pa{\frac{\mu_h^\dagger(a_h|x_h)}{\mu_{h}^{t+1}(a_h|x_h)}}+\frac{1}{\eta_h^t(x_h)}\log\pa{\frac{\mu_h^\dagger(a_h|x_h)}{\mu_h^t(a_h|x_h)}}\right.\\
        &\qquad+\left.\pa{\frac{1}{\eta_h^{t+1}(x_h)}-\frac{1}{\eta_h^{t}}(x_h)}\log\pa{\frac{\mu_a^\dagger}{\mu_{h}^0(x_h^t,a_h)}}\right]
        +\sum_{t=1}^T\frac{1}{\eta_{1}^{t+1}(x_1^t)}\log(Z_1^t)\\
        &=-\mathcal{D}_{\eta^{T+1}}\pa{\mu^\dagger,\mu^{T+1}}+\mathcal{D}_{\eta^{T+1}}\pa{\mu^\dagger,\mu^0}+\sum_{t=1}^T\frac{1}{\eta_{1}^{t+1}(x_1^t)}\log(Z_1^t)\\
        &\leq \mathcal{D}_{\eta^{T+1}}\pa{\mu^\dagger,\mu^0}+\sum_{t=1}^T\frac{1}{\eta_{1}^{t+1}(x_1^t)}\log(Z_1^t)\,.
    \end{align*}
    The rest of the lemma then follows by the decomposition
    
    \[\scal{\mu_{1:}^t-\w_{1:}^\dagger}{\ell^t}=\scal{\mu_{1:}^t}{\ell^t-\tell^t}+\scal{\w_{1:}^\dagger}{\tell^t-l^t}+\scal{\mu_{1:}^t-\w^\dagger}{\tell^t}\]
    
    and the expression of the regret as a scalar product.
\end{proof}

The following lemma gives a bound with high probability on the two \textrm{BIAS} terms for \AdaptiveFTRL.

\begin{lemma}\label{lemma:bias_ada}
Under event $\eventadapt$, assuming $\gamma\geq 1$, we have
    \[\textrm{BIAS I}\leq H\sqrt{2T\iota'}+\gamma\logtwo\AX\,, \qquad \textrm{BIAS II}\leq X+\frac{2\iota'}{\gamma}HT\,.\]
\end{lemma}

\begin{proof}
    We first decompose BIAS I into two terms,
    \[\sum_{t=1}^T\scal{\mu_{1:}^t}{\ell^t-\tell^t}=\sum_{t=1}^T\scal{\mu_{1:}^t}{\ell^t-\hl^t}+\sum_{t=1}^T\scal{\mu_{1:}^t}{\hl^t-\tell^t}\,.\]
    
    The first term is directly upper-bounded by $H\sqrt{2T\iota'}$ using the inequality \eqref{azuma_ineq}. For the second term, we have
    
    \begin{align*}
        \sum_{t=1}^T\scal{\mu_{1:}^t}{\hl^t-\tell^t}&=\sum_{t=1}^T\sum_{h=1}^H\sum_{(x_h,a_h)\in\cAXh}\indic{x_h=x_h^t,a_h=a_h^t} \:\mu_{1:h}^t(x_h,a_h)\pa{\frac{1-r_h^t}{\mu_{1:h}^t(x_h,a_h)}-\frac{1-r_h^t}{\mu_{1:h}^t(x_h,a_h)+\gamma_h^t(x_h,a_h)}}\\
        &=\sum_{t=1}^T\sum_{h=1}^H\sum_{(x_h,a_h)\in\cAXh}\indic{x_h=x_h^t,a_h=a_h^t}(1-r_h^t)\frac{\gamma_h^t(x_h,a_h)}{\mu_{1:h}^t(x_h,a_h)+\gamma_h^t(x_h,a_h)}\\
        &\leq \sum_{t=1}^T\sum_{h=1}^H\sum_{(x_h,a_h)\in\cAXh}\tp_{1:h}^t(x_h,a_h)\gamma_h^t(x_h,a_h)\\
        &=\gamma\sum_{h=1}^H\sum_{(x_h,a_h)\in\cAXh}\sum_{t=1}^T\frac{\tp_{1:h}^t(x_h,a_h)}{1+\tP_{1:h}^{t-1}(x_h,a_h)}\\
        &\leq_1 \gamma \sum_{h=1}^H\sum_{(x_h,a_h)\in\cAXh}\log_2(1+\tP_{1:h}^T(x_h,a_h))\\
        &\leq_2 \gamma \sum_{h=1}^H\sum_{(x_h,a_h)\in\cAXh}\log_2(2+2P_{1:h}^T(x_h,a_h))\\
        &\leq \gamma \logtwo\AX
    \end{align*}
    
    as $\logtwo=1+\log_2(1+T)$ by definition, where we used Lemma~\ref{lemma_logbound} for $\leq_1$, as $\tp_{1:h}^t(x_h,a_h)\leq \frac{1}{\gamma}(1+\tP_{1:h}^{t-1}(x_h,a_h))\leq 1+\tP_{1:h}^{t-1}(x_h,a_h)$ and inequality \eqref{p_ineq} for $\leq_2$.
    
    In order to bound $\textrm{BIAS II}$, we use inequality \eqref{l_ineq} and get for all $\w^\dagger\in\maxpi$,
    
    \[\scal{\w^\dagger_{1:}}{\tL^T-L^T}\leq \sum_{h=1}^H\sum_{(x_h,a_h)\in\cAXh}\w_{1:h}^\dagger(x_h,a_h)+\frac{2\iota'}{\gamma}\scal{\mu_{1:}^\dagger}{P_{1:}^T}\leq X+\frac{2\iota'}{\gamma}HT\,.\]
\end{proof}

Contrary to the analysis of the first algorithm, the bound on the \textrm{REG} term now only holds with high probability because of the adaptive learning rate.

\begin{lemma}
    Under the event $\eventadapt$, we have the bound
    \[\textrm{REG}\leq \frac{2H\log(\AX)}{\eta} (X+T)\,.\]
\end{lemma}

\begin{proof}
    Using inequality \eqref{p_ineq} and the fact that $\mu^0$ is the uniform distribution, we get for any $\mu^\dagger\in\maxpi$
    \begin{align*}
        \mathcal{D}_{\eta^{T+1}}\pa{\mu^\dagger,\mu^0}&=\sum_{h=1}^H\sum_{(x_h,a_h)\in\cAXh}\frac{1}{\eta_h^{T+1}(x_h)}\mu^{\dagger}_{1:h}(x_h,a_h)\log\pa{\frac{\mu_h^\dagger(a_h|x_h)}{\mu_h^0(a_h|x_h)}}\\
        &\leq \frac{1}{\eta}\sum_{h=1}^H\sum_{x_h\in\cX_h}\max_{x'_{h'}\geq x^{}_h}\bra{1+\tP_{1:h'}^T(x'_{h'})}\sum_{a_h\in\cAs{x}} \mu_{1:h}^\dagger(x_h,a_h)\log(\Ax)\\
        &\leq \frac{\log(\AX)}{\eta}\sum_{h=1}^H\sum_{x_h\in\cX_h}\max_{x'_{h'}\geq x^{}_h}\bra{1+\frac{1}{\Ax}\sum_{a'_{h'}\in\cA{x_{h'}}}\pa{1+\tP_{1:h'}^T(x'_{h'},a'_{h'})}}\sum_{a_h\in\cAs{x}}\mu_{1:h}^\dagger(x_h,a_h)\\
        &\leq \frac{\log(\AX)}{\eta}\sum_{h=1}^H\sum_{x_h\in\cX_h}\max_{x'_{h'}\geq x^{}_h}\bra{2+2P_{1:h'}^T(x'_{h'})}\sum_{a_h\in\cAs{x}}\mu_{1:h}^\dagger(x_h,a_h)\\
        &\leq \frac{2\log(\AX)}{\eta}\sum_{h=1}^H\sum_{x_h\in\cX_h}\bra{1+P_{1:h}^T(x_{h})}\sum_{a_h\in\cAs{x}}\mu_{1:h}^\dagger(x_h,a_h)\\
        &\leq \frac{2H\log(\AX)}{\eta} (X+T)
    \end{align*}
    and we concludes taking the max over $\mu^\dagger\in\maxpi$.
\end{proof}

Upper bounding the \textrm{VAR} term is however more difficult. We give a preliminary lemma that requires the assumption that the individual learning rates of the information states at any given time increase along a trajectory. This hypothesis is trivially satisfied by \AdaptiveFTRL thanks to the particular definition of $\eta_h^t$.

\begin{lemma}\label{lemma_varstep}
    For all $t\in [T]$, if the learning rates at time $t$ increase along the trajectory, then
    
    \[\scal{\mu_{1:}^t}{\tell^t}+\frac{1}{\eta_1^{t+1}(x_1^t)}\log\pa{Z_1^t}\leq \frac{H}{2}\sum_{h=1}^H \eta_h^t(x_h^t)\tell_h^t(x_h^t,a_h^t)\,.\]
\end{lemma}

\begin{proof}
    For readability we will define for all $h\in[H]$ in this proof, $\omu_h^t:=\mu_h^t(a_h^t|x_h^t)$, $\oeta_h^t:=\eta_h^t(x_h^t)$, $\kappa_h^t:=\oeta_h^t/\oeta_{h+1}^t$ and $\tell_h^t:=\tell_h^t(x_h^t,a_h^t)$.
    
    Using Hölder inequality, we first get:

    \begin{align*}
        Z_h^t&=\sum_{a_h\in\cAs{x_h^t}}\mu_h^t(a_h|x_h^t)^{\alpha_h^t}\mu_h^0(a_h |x_h^t)^{1-\alpha_h^t}\exp\bra{-\indic{a_h^t=a_h}\eta_h^{t+1}(x_h^t)\lell_h^t}\\
        &\leq \pa{\sum_{a_h\in\cAs{x_h^t}}\mu_h^t(a_h|x_h^t)) \exp\bra{-\indic{a_h=a_h^t}\eta_h^{t+1}(x_h^t)\lell_h^t/\alpha_h^t}}^{\alpha_h^t}\pa{\sum_{a_h\in\cAs{x_h^t}}\mu_h^0(a_h|x_h^t))}^{1-\alpha_h^t}\\
        &=\pa{1-\omu_h^t+\omu_h^t\exp\bra{-\eta_h^{t+1}(x_h^t)\lell_h^t/\alpha_h^t}}^{\alpha_h^t}\,.
    \end{align*}
    
    This yields as $\lell_h^t=\tell_h^t-\log(Z_{h+1}^t)/\eta_{h+1}^{t+1}(x_{h+1}^{t})$ and $\alpha_h^t=\eta_h^{t+1}(x_h^t)/\oeta_h^t$, using convexity and $\kappa_h^t\leq 1$,
    
    \begin{align*}
        (Z_h^t)^{1/\alpha_h^t}&\leq 1-\omu_h^t+\omu_h^t\pa{Z_{h+1}^t}^{\oeta_h^t/\eta_{h+1}^t(x_{h+1}^{t})}\exp\pa{-\oeta_h^t\tell_h^t}\\
        &=1-\omu_h^t+\omu_h^t\pa{Z_{h+1}^t}^{\kappa_h^t/\alpha_{h+1}^t}\exp\pa{-\oeta_h^t\tell_h^t}\\
        &\leq 1-\omu_h^t+\omu_h^t\pa{\kappa_h^t\pa{Z_{h+1}^t}^{1/\alpha_{h+1}^t}+(1-\kappa_h^t)}\exp\bra{-\oeta_h^t\tell_h^t}\,.
    \end{align*}
    
    We can then define the sequence $W_h^t$ such that $W_{H+1}^t:=1$, and for which the previous relation is always an equality, i.e.
    
    \[W_{h}^t:= 1-\omu_h^t+\omu_h^t\pa{\kappa_h^t W_{h+1}^t+(1-\kappa_h^t)}\exp\bra{-\oeta_h^t\tell_h^t}\]
    
    and we can get by induction, for all $h\in[H]$, $W_{h}^t\geq (Z_h^t)^{1/\alpha_h^t}$. If we introduce the independent auxiliary variables $y_h^t\sim \cB(\kappa_h^t)$ and $z_h^t\sim \cB(\omu_h^t)$, with the associated products $y^t_{h:h'}=\Pi_{h\leq i\leq h'}y^t_i$ and $z^t_{h:h'}=\Pi_{h\leq i\leq h'}z^t_i$, we recursively get 
    
    \[ W_h^t=\E_{y^t,z^t}\bra{\exp\pa{-\sum_{h\leq h'\leq H}y^t_{h:h'-1}z^t_{h:h'}\oeta_{h'}^t\tell_{h'}^t}}\,.\]
    
    Finally, $\log(x)\leq -1+x$, $e^{-x}\leq 1-x+x^2/2$ and the above relation leads to
    
    \begin{align*}
        \log(Z_1^t)&\leq \alpha_1^t \log(W_1^t)\\
        &=\alpha_1^t\log\pa{\E_{y^t,z^t}\bra{\exp\pa{-\sum_{h=1}^H y^t_{1:h-1}z^t_{1:h}\oeta_{h}^t\tell_{h}^t}}}\\
        &\leq \alpha_1^t \pa{-1+\E_{y^t,z^t}\bra{\exp\pa{-\sum_{h=1}^H y^t_{1:h-1}z^t_{1:h}\oeta_{h}^t\tell_{h}^t}}}\\
        &\leq-\alpha_1^t\E_{y^t,z^t}\bra{\sum_{h=1}^H y^t_{1:h-1}z^t_{1:h}\oeta_{h}^t\tell_{h}^t}+\frac{1}{2}\alpha_1^t\E_{y^t,z^t}\bra{\pa{\sum_{h=1}^H y^t_{1:h-1}z^t_{1:h}\oeta_{h}^t\tell_{h}^t}^2}\\
        &\leq -\alpha_1^t\sum_{h=1}^H \oeta_1^t \omu_{1:h}^t \tell_h^t+\frac{1}{2}\alpha_1^t\E_{y^t,z^t}\bra{H\sum_{h=1}^H y^t_{1:h-1}z^t_{1:h}\pa{\oeta_h^t \tell_h^t}^2}\\
        &= -\eta_1^{t+1}(x_1^t)\sum_{h=1}^H \omu_{1:h}^t \tell_h^t +\frac{H}{2}\alpha_1^t\sum_{h=1}^H \oeta_1^t\omu_{1:h}^t \oeta_h^t \pa{\tell_h^t}^2\\
        &\leq -\eta_1^{t+1}(x_1^t)\sum_{h=1}^H \omu_{1:h}^t \tell_h^t +\frac{H}{2}\eta_1^{t+1}(x_1^t)\sum_{h=1}^H \oeta_h^t \tell_h^t
    \end{align*}
    
    where we used $\tell_h^t\omu_{1:h}^t\leq 1$ for the last inequality, and we finally obtain
    
    \[\sum_{h=1}^H \omu_{1:h}^t \tell_h^t + \frac{1}{\eta_1^{t+1}(x_1^t)}\log(Z_1^t)\leq \frac{H}{2}\sum_{h=1}^H \oeta_h^t \tell_h^t\]
    
    which concludes the proof.
\end{proof}

We then directly use this lemma on all episodes to upper bound  the \textrm{VAR} term.
\begin{lemma}
    Under the event $\eventadapt$, assuming $\gamma\geq 1$, we have the bound
    \[\textrm{VAR}\leq \frac{\eta}{2}H \AX \logtwo\,.\]
\end{lemma}

\begin{proof}
    Thanks to Lemma~\ref{lemma_varstep}, we have
    
    \begin{align*}
        \textrm{VAR}&\leq \frac{H}{2}\sum_{t=1}^T\sum_{h=1}^H \eta_h^t(x_h^t)\tell_h^t(x_h^t,a_h^t)\\
        &\leq \frac{H}{2}\sum_{t=1}^T\sum_{h=1}^H\sum_{(x_h,a_h)\in\cAXh}\eta_h^t(x_h)\tp_{1:h}^t(x_h,a_h)\\
        &= \frac{\eta H}{2}\sum_{t=1}^T\sum_{h=1}^H\sum_{x_h\in\cX_h} \frac{\tp_{1:h}^t(x_h,a_h)}{1+\tP_{1:h}^{t-1}(x_h)}\\
        &= \frac{\eta H}{2}\sum_{h=1}^H\sum_{x_h\in\cX_h}\sum_{t=1}^T \As{x_h}\frac{\tp_{1:h}^t(x_h)}{1+\tP_{1:h}^{t-1}(x_h)}\\
        &\leq \frac{\eta H}{2}\sum_{h=1}^H\sum_{x_h\in\cX_h}\As{x_h}\log_2\pa{1+\tP_{1:h}^T(x_h)}\\
        &=\frac{\eta H}{2}\sum_{h=1}^H\sum_{x_h\in\cX_h}\As{x_h}\log_2\pa{1+\frac{1}{\As{x_h}}\sum_{a_h\in\cAs{x_h}}\tP^T_{1:h}(x_h,a_h)}\\
        &\leq \frac{\eta H}{2}\sum_{h=1}^H\sum_{x_h\in\cX_h}\As{x_h}\log_2\pa{1+\frac{1}{\As{x_h}}\sum_{a_h\in\cAs{x_h}}(1+2P^T_{1:h}(x_h,a_h))}\\
        &\leq \frac{\eta H}{2}\sum_{h=1}^H\sum_{x_h\in\cX_h}\As{x_h}\log_2\pa{2+2T}\\
        &=\frac{\eta\logtwo H}{2} \AX
    \end{align*}
    
    where we used $\tp_{1:h}^t(x_h^t)\leq \eta_h^t(x_h^t)/A(x_h^t)\leq (1+\tP_{1:h}^{T-1})/(\eta A(x_h^t))\leq 1+\tP_{1:h}^{T-1}$ for the assumption of Lemma~\ref{lemma_logbound}, and inequality \ref{p_ineq} for the next to last inequality.
\end{proof}

We now have all the tools needed to upper bound the regret with high probability when using \AdaptiveFTRL.

\adatheorem*

\begin{proof}
    We first notice that if $T<4\AX$, then, as $\regret_\mathrm{max}^T\leq T$, we have
    
    \[\regret_\mathrm{max}^T< 2\sqrt{\AX T}\]
    
    and the bound is then immediate. Similarly, if $\gamma< 2\iota'$, then
    
    \[\regret_\mathrm{max}^T\leq \sqrt{\frac{2\iota'\logtwo\AX T}{H}}\]
    
    and the bound also trivially holds. We will thus assume in the rest of the proof that $\AX\leq T$ and $\gamma \geq 2\iota'$. This last inequality lets us use Lemma~\ref{lemma:approx_ada} under event $\epsilon$, as it also implies $\gamma \geq 1$.
    
    In this case, using the previous lemmas, along with  $2\leq \AX\leq T$, $2\leq \iota'$, $2\log(\AX)\leq \iota'$ and $2\leq \logtwo$ for the last two inequalities:
    
    \begin{align*}
        \regret_{\mathrm{max}}^T&\leq \textrm{BIAS I}+\textrm{BIAS II}+\textrm{REG}+\textrm{VAR}\\
        &\leq H\sqrt{2T\iota'}+X+\gamma\logtwo\AX+\frac{2\iota'}{\gamma}HT+\frac{2\log(\AX)}{\eta}H(X+T)+\frac{\eta\logtwo}{2}H\AX\\
        &\leq (\frac{1}{\sqrt{2}}+\frac{1}{4})H\sqrt{\iota'\logtwo\AX T}+\gamma\logtwo\AX+\frac{2\iota'}{\gamma}HT+\frac{4\log(\AX)}{\eta}HT+\frac{\eta\logtwo}{2}H\AX\\
        &\leq H\sqrt{\iota'\logtwo\AX T}+\gamma\logtwo\AX+\frac{2\iota'}{\gamma}HT+\frac{2\iota'}{\eta}HT+\frac{\eta\logtwo}{2}H\AX\,.
    \end{align*}
    
    The previous expression is minimized with $\gamma=\sqrt{2\iota' HT/(\logtwo\AX)}$ and $\eta=2\sqrt{\iota' T/(\logtwo\AX)}$, which yields
    
    \begin{align*}
        \regret^T(\mu^\star)&\leq H\sqrt{\iota'\logtwo\AX T}+2H\sqrt{2\iota'\logtwo \AX T}+2H\sqrt{\iota'\logtwo \AX T}\\
        &\leq 6H\sqrt{\iota'\logtwo \AX T}
    \end{align*}
    where we used $2\sqrt{2}\leq 3$.
    
\end{proof}